%% file: main.tex
\documentclass{CVM}

\usepackage{algorithmic}
\usepackage{amsthm}
\usepackage{mathtools}
\usepackage{overpic}
\usepackage{placeins}
\graphicspath{{images/}}
\newcommand{\resultsNA}{--}

\newtheorem{theorem}{Theorem}

\theoremstyle{remark}
\newtheorem{remark}{Remark}

\CVMsetup{
  type      = {Research Article},
  title     = {Hybrid-Basis Feature Forecasting for Diffusion Sampling Acceleration},
  author    = {Kai-Liang Cheng\textsuperscript{1},
               Yuan-Yuan Cheng\textsuperscript{1},
               Yu-fan Jin\textsuperscript{1}, and
               Xiao-Ming Fu\textsuperscript{1}\\[0.45em]
               \normalfont\fontsize{9pt}{11pt}\selectfont
               \textsuperscript{1} University of Science and Technology of China, China},
  runauthor = {K.-L. Cheng, Y.-Y. Cheng, Y.-F. Jin, and X.-M. Fu},
  runtitle  = {Hybrid-Basis Feature Forecasting for Diffusion Sampling Acceleration},
  abstract  = {We propose Hybrid-Basis Feature Forecasting (\emph{HybridFF}), a training-free, plug-and-play framework for accelerating diffusion sampling. To capture local smoothness, long-range trends, and complex non-monotonic variations when modeling feature evolution, HybridFF first estimates coefficients using moving least squares (MLS) for each of multiple complementary basis families and then combines the corresponding predictors using fusion weights. In addition to the choice of basis functions, the fusion weights also play a critical role. We introduce two strategies to balance quality and speedup. HybridFF (Fixed) prioritizes efficiency with model-specific fusion weights calibrated on a small set and held constant during inference. HybridFF (Adaptive) updates the fusion weights online using branch reliability scores computed from an exponential moving average of full-step prediction errors, improving prediction fidelity and generation quality under aggressive caching while retaining substantial acceleration. Experiments across DiT-XL/2, FLUX.1-dev, SD3.5-Large, and HunyuanVideo demonstrate a favorable speedup--quality trade-off over representative single-basis forecasters and caching baselines.},
  keywords  = {diffusion models, sampling acceleration, feature forecasting, moving least squares}
}

\begin{document}

\raggedbottom
\setlength{\parskip}{0pt}
\setlength{\parindent}{1em}
\titlespacing*{\paragraph}{0pt}{0.55em plus 0.15em minus 0.10em}{0.5em}

\maketitle

\input{src_combination/introduction}

\input{src_combination/relatedwork}
\input{src_combination/method}
\input{src_combination/results}
\input{src_combination/conclusion}
\FloatBarrier

\bibliographystyle{CVMbib}
\bibliography{src_combination/reference}

\end{document}

%% file: src_combination/introduction.tex
\section{Introduction}
\label{sec:introduction}

\begin{figure*}[!ht]
    \centering
    \includegraphics[width=0.8\textwidth]{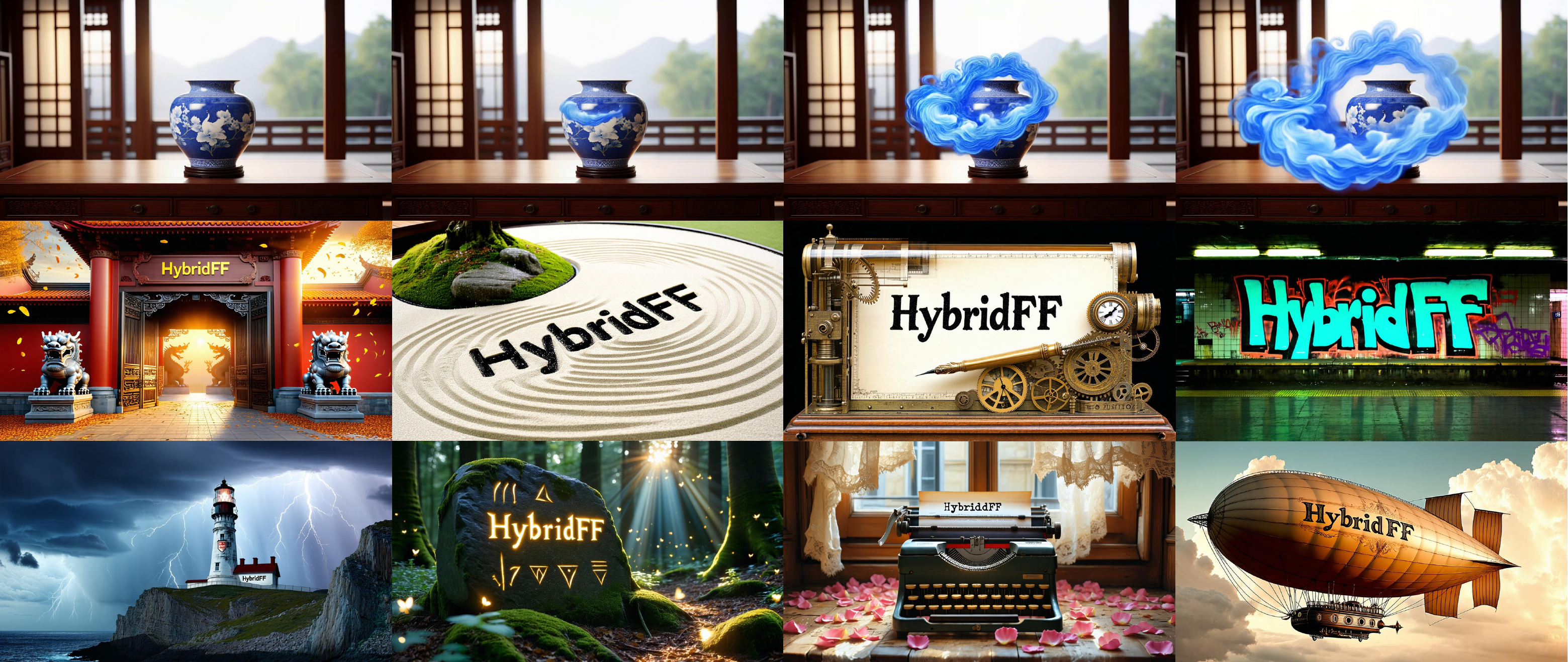}
    \caption{
    HybridFF results: HunyuanVideo (top row, $3.85\times$) and FLUX.1-dev (bottom two rows, $4.37\times$) at NFE\,$=$\,10.
    }
    \label{fig:teaser}
\end{figure*}

Diffusion models~\cite{ho2020ddpm}, particularly those leveraging the Diffusion Transformer (DiT) architecture \cite{peebles2023scalable}, have redefined the state of the art in generative modeling across multiple modalities, including high-fidelity image synthesis \cite{chen2024pixart} and video generation \cite{blattmann2023}. 
However, the iterative nature of the denoising process remains a formidable barrier to efficient deployment. 
Each sampling step involves redundant, computationally expensive evaluations of the denoising network via deep stacks of attention and MLP modules, resulting in prohibitive inference latency \cite{liu2025teacache,zhu2026tap}. 
Consequently, developing efficient sampling acceleration techniques has become an important topic. 

Current acceleration methods can be broadly categorized into sampling step reduction \cite{lu2022dpm,song2023consistency} and per-step computation optimization \cite{li2024faster,fang2023pruning}. 
Among these, feature caching has emerged as a compelling training-free, plug-and-play solution that exploits temporal redundancy between adjacent timesteps \cite{ma2024deepcache,liu2025teacache}. 
This approach has evolved from \emph{static reusing} methods like DeepCache \cite{ma2024deepcache} and TeaCache \cite{liu2025teacache}, which assume feature invariance across steps, to \emph{dynamic forecasting} frameworks exemplified by TaylorSeer \cite{liu2025taylorseer}. 


Consistent with prior observations~\cite{liu2025taylorseer}, Fig.~\ref{fig:mls_prediction}(a) shows that intermediate diffusion features, including transformer attention/MLP outputs, follow structured temporal trends over local cache intervals rather than changing arbitrarily across adjacent denoising steps.

\begin{figure*}[t]
    \centering
    \includegraphics[width=0.95\textwidth]{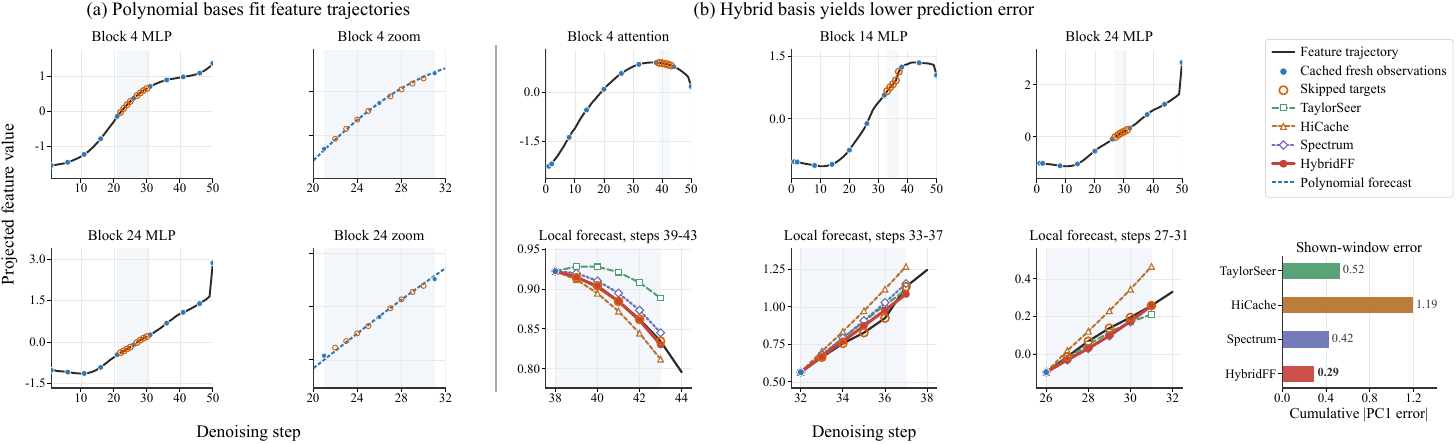}
    \caption{
    DiT-XL/2 feature-trajectory analysis. 
    (a) Intermediate features follow structured trajectories along denoising steps, and full-step observations can be fitted by polynomial bases to forecast skipped-step features. 
    (b) Hybrid-basis polynomial forecasting gives lower feature-prediction error than single-basis predictors, motivating HybridFF's coefficient-level fusion.}
    \label{fig:mls_prediction}
\end{figure*}

Motivated by these predictable structures, many dynamic forecasting methods have been proposed.
%
TaylorSeer \cite{liu2025taylorseer} uses polynomial bases, which often deviate around sharp transitions, 
while HiCache \cite{feng2026hicache} and Spectrum \cite{han2026spectrum} introduce Hermite and Chebyshev polynomial bases, respectively, to improve the stability of feature forecasting under complex or long-range temporal dynamics.
Similarly, TC-Pad\'{e} \cite{cui2026tcpade} introduces rational functions to better capture asymptotic behaviors. 
Despite their success, these forecasting methods typically rely on a \emph{single basis family} or a uniform temporal prior to model feature dynamics.

In practice, diffusion features exhibit highly heterogeneous dynamics across different timesteps, noise levels, network layers, and individual samples \cite{cai2026lesa,zhu2026tap,cui2026tcpade}. 
A single basis inevitably imposes a rigid inductive bias that cannot adapt to the complex, non-monotonic variations inherent in the denoising process. 
This \emph{one-basis-fits-all} dilemma restricts the generation quality at aggressive acceleration ratios, leading to larger feature-prediction errors than hybrid-basis forecasting (Fig.~\ref{fig:mls_prediction}(b)).

To overcome this limitation, we propose \textbf{Hybrid-Basis Feature Forecasting (HybridFF)}, a novel training-free framework that models diffusion feature evolution as a weighted composition of multiple complementary basis families. 
The core insight of HybridFF is that different mathematical bases possess unique strengths, e.g.,
Taylor expansions \cite{liu2025taylorseer} are effective for modeling local smoothness,
Hermite polynomials \cite{feng2026hicache} provide stable interpolation with derivative-aware constraints, and
Chebyshev polynomials \cite{han2026spectrum} offer robust approximation over broader intervals.
To combine these bases into a unified framework that accommodates the multifaceted dynamics of the diffusion model's features across the entire trajectory, we first estimate the coefficients for each basis using moving least squares (MLS) and then apply fusion weights to combine the corresponding predictors.
%
We further introduce two practical weighting strategies to balance speed and quality. 
HybridFF (Fixed) calibrates fusion weights on a small set and keeps them constant during inference, emphasizing the quality--speedup trade-off. 
HybridFF (Adaptive) updates fusion weights based on full-step prediction errors, maintaining high fidelity under aggressive caching.

Extensive evaluations and comparisons with state-of-the-art models (DiT-XL/2, FLUX.1-dev, SD3.5-Large, and HunyuanVideo) show that HybridFF achieves a superior speed-quality trade-off (Tables~\ref{tab:dit_main}--\ref{tab:t2v_main}), outperforming single-basis baselines by a significant margin. For example, HybridFF (Fixed) achieves up to $2.44\times$ speedup on DiT-XL/2, $4.37\times$ on FLUX.1-dev, $4.75\times$ on SD3.5-Large, and $3.85\times$ on HunyuanVideo. Representative image and video results with NFE\,$=$\,10 are shown in Fig.~\ref{fig:teaser}.

%% file: src_combination/relatedwork.tex
\section{Related Works}\label{sec:relatedworks}

\paragraph{Diffusion Models and Generative Applications.}

Diffusion models have become a dominant paradigm for high-quality generative modeling across image synthesis, conditional generation, and video generation \cite{ho2020ddpm,rombach2022ldm,saharia2022imagen,podell2023sdxl,peebles2023scalable,chen2024pixart,esser2024flux,kong2024hunyuanvideo}. 
Among them, diffusion transformers (DiTs) significantly improve scalability and generation quality by replacing convolutional backbones with transformer architectures \cite{peebles2023scalable}. 
These advances have enabled diffusion-based systems to be applied to a broad range of tasks, including photorealistic text-to-image generation \cite{saharia2022imagen,podell2023sdxl,chen2024pixart,esser2024flux}, video generation \cite{kong2024hunyuanvideo,ho2022videodiffusion,singer2022makeavideo}, and restoration tasks such as super-resolution and inpainting \cite{saharia2021srdiff,lugmayr2022repaint}. 
However, despite their strong generation capability, these models remain computationally expensive at inference time due to the iterative denoising process, which motivates efficient diffusion inference methods.

\paragraph{Diffusion Sampling Acceleration.}

Existing methods for accelerating diffusion inference mainly follow two directions. 
The first reduces the number of denoising steps through improved numerical solvers, consistency-style learning, or distillation-based trajectory compression. 
Representative examples include DDIM \cite{song2021ddim}, DPM-Solver \cite{lu2022dpm}, consistency models \cite{song2023consistency}, and progressive distillation \cite{salimans2022progressive}. 
The second reduces the computational cost of each denoising step through model simplification, structural pruning, quantization, token reduction, or architectural acceleration \cite{li2024faster,fang2023pruning,shang2023ptqdiffusion,li2023qdiffusion,bolya2023tokenmerging,yuan2024ditfastattn}. 
Among these approaches, training-free methods that avoid retraining the generative model are particularly attractive, since they can better preserve the original model behavior and are often applicable in a plug-and-play manner.

\paragraph{Feature Caching and Feature Forecasting.}

Feature caching has recently emerged as an effective training-free paradigm for diffusion acceleration. 
Early methods mainly follow a \emph{cache-then-reuse} strategy, where intermediate activations computed at previous timesteps are directly reused to skip redundant computation in later steps \cite{ma2024deepcache,selvaraju2024fora,zou2025toca,liu2025smoothcache,liu2025teacache}. 
While exploiting temporal feature similarity enables low-overhead speedups, direct reuse becomes unreliable at large skip intervals, leading to error accumulation and quality degradation.
%
\begin{figure*}[t]
    \centering
 \begin{overpic}[width=0.88\linewidth]{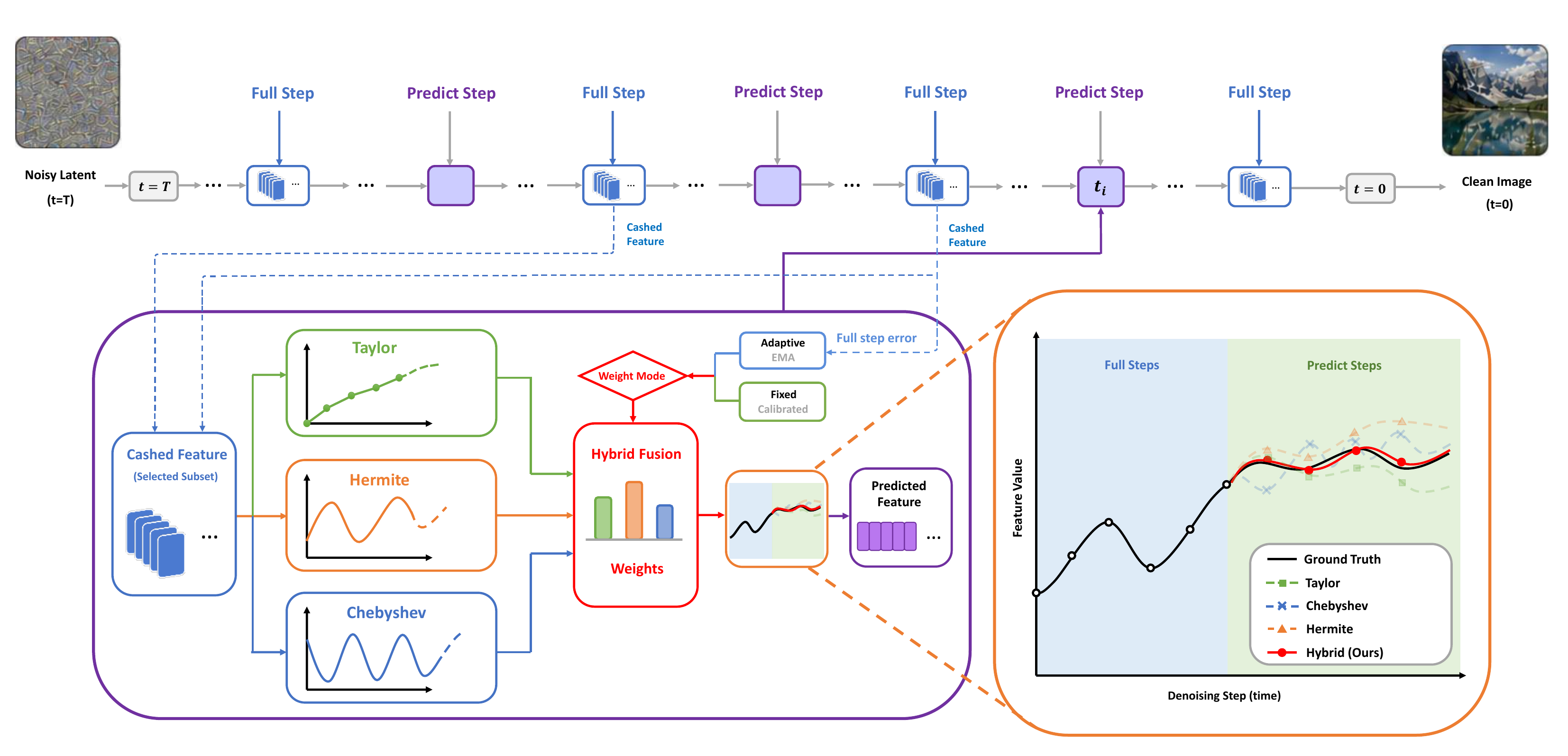}
 \end{overpic}
    \caption{Overview of the proposed hybrid-basis forecasting pipeline. 
    The pipeline proceeds in three stages: 
    (1) full steps compute and cache exact intermediate features; 
    (2) predict steps fit Taylor, Hermite, and Chebyshev predictors from the same normalized history; 
    (3) the branch coefficient vectors are fused by simplex weights and used to reconstruct the de-normalized target block feature.}
    \label{fig:pipeline}
\end{figure*}
Recent works therefore shift to \emph{cache-then-forecast}, predicting future features from historical activations rather than copying them.
TaylorSeer \cite{liu2025taylorseer} frames forecasting as feature extrapolation, establishing it as a stronger alternative to naive reuse. 
Subsequent works improve forecasting via Adams-Bashforth-style extrapolation \cite{yu2025abcache}, plug-in predictor upgrades \cite{zheng2026foca}, spectral or global trajectory modeling \cite{han2026spectrum}, and trajectory-consistent rational approximation \cite{cui2026tcpade}. 
Other extensions further incorporate adaptive forecasting mechanisms \cite{liu2025speca,chen2026svdcache,cai2026lesa,zou2026disca,sun2025unicp}.
Together, they show that forecasting-based acceleration offers better speed--quality trade-offs than direct reuse.

\paragraph{Single-Basis and Hybrid-Basis Forecasting.}

Despite diverse formulations, most forecasting methods still rely on a single prior or basis family, ranging from local extrapolation and global/spectral fitting to adaptive, speculative, subspace-aware, or learnable predictor designs \cite{liu2025taylorseer,han2026spectrum,feng2026hicache,zheng2026foca,cui2026tcpade,cai2026lesa,wang2026zeus,zhu2026tap,liu2025speca,chen2026svdcache,zou2026disca}. 
While effective in their target regimes, a single predictor form imposes a restrictive bias: local predictors may struggle with long-range evolution, whereas global ones may miss local trend changes. 
Since diffusion features exhibit diverse temporal patterns across timesteps, layers, and samples, relying on a single basis family can limit prediction flexibility under aggressive skipping ratios.
%
In contrast, we explore \emph{hybrid-basis forecasting}, which combines multiple complementary basis families. HybridFF therefore captures local behavior, long-range trends, and complex non-monotonic feature evolution during diffusion sampling. 

%% file: src_combination/method.tex
\section{Method} \label{sec:method}


\subsection{Background}\label{sec:method:background}

\paragraph{Diffusion model.} Diffusion models generate data by learning to reverse a noise corruption process. The forward process gradually adds Gaussian noise to a clean sample $\mathbf{x}_0$ drawn from the data distribution $p_{\text{data}}$ over $T$ steps: $q(\mathbf{x}_t \mid \mathbf{x}_0) = \mathcal{N}(\mathbf{x}_t; \sqrt{\bar{\alpha}_t} \mathbf{x}_0, (1 - \bar{\alpha}_t) \mathbf{I})$,
where $\bar{\alpha}_t$ controls the noise schedule and $\mathbf{x}_T \approx \mathcal{N}(0, \mathbf{I})$. The reverse process learns to denoise by training a neural network $\epsilon_\theta$ to predict the added noise:
$\mathcal{L} = \mathbb{E}_{t, \mathbf{x}_0, \epsilon} \left[ \lVert \epsilon - \epsilon_\theta(\mathbf{x}_t, t) \rVert^2 \right].$
At inference, the model iteratively denoises from $\mathbf{x}_T$ to $\mathbf{x}_0$ over $N$ steps via a reverse sampler: $\mathbf{x}_{t_{i-1}} = f(\mathbf{x}_{t_i}, \epsilon_\theta(\mathbf{x}_{t_i}, t_i), t_i, t_{i-1}),$
where $f$ depends on the sampler (e.g., DDPM~\cite{ho2020ddpm}, DDIM~\cite{song2021ddim}, DPM-Solver~\cite{lu2022dpm}).

\paragraph{Problem Formulation.}\label{sec:method:setup}
 Consider a diffusion model with $L$ modules evaluated over $N$ denoising timesteps $\{t_0, t_1, \dots,
  t_{N-1}\}$. At each timestep $t_i$ and module $\ell \in \{1, \dots, L\}$, let $z_{t_i}^\ell \in \mathbb{R}^{d}$ denote an intermediate feature (e.g., attention output in transformer-based models). 
   For brevity, we drop the layer superscript when the context is clear.
   We partition the timesteps into a \emph{full} set $\mathcal{U}$ where the model is evaluated exactly and a \emph{predict} set $\mathcal{V} = \{t_0, \dots, t_{N-1}\} \setminus \mathcal{U}$ where the feature is reconstructed from caches.


For each predict step $t_i \in \mathcal{V}$, the visible history is the ordered index set $\mathcal{I}_{t_i} = \{t_{i,1} < t_{i,2} < \cdots < t_{i,K}\} \subset \mathcal{U}, \quad t_{i,K} < t_i,$
yielding the cached buffer $\mathcal{H}_{t_i} = \big\{(t_{i,j},\, z_{t_{i,j}})\big\}_{j=1}^{K}.$
The forecasting objective is to produce an estimator $\hat{z}_{t_i} = \mathcal{P}(\mathcal{H}_{t_i}, t_i)$ with low reconstruction error, while keeping the predictor lightweight relative to a full model evaluation. 
Following common practice~\cite{liu2025taylorseer,han2026spectrum}, the scheduler that determines $\mathcal{U},\mathcal{V}$ and the predictor $\mathcal{P}$ is decoupled, so different schedules can share the same predictor.

\subsection{Unified Coefficient Forecasting via MLS}
\label{sec:method:mls}

\paragraph{Formulation}
As observed in Fig.~\ref{fig:mls_prediction}(a), local feature trajectories are structured and can be approximated by polynomial bases, reducing skipped-feature forecasting to coefficient estimation from cached observations. At each prediction step $t_i$, HybridFF reconstructs the skipped feature as a linear combination of cached features:
\begin{equation}
\hat{z}_{t_i} = \sum_{j=1}^{K} \alpha_{t_{i,j}}\, z_{t_{i,j}},
\label{eq:linear-combination}
\end{equation}
where $\alpha_{t_i} = (\alpha_{t_{i,1}}, \dots, \alpha_{t_{i,K}})^{\!\top} \in \mathbb{R}^{K}$ is a coefficient vector. 
The key question is \emph{how to determine $\alpha_{t_i}$}. 

We instantiate two solver families: for the Taylor branch,
$\alpha_{t_i}^{(\mathrm{Taylor})}$ is given in closed form by finite-difference
extrapolation weights without any matrix solve \cite{liu2025taylorseer} (see Sec.~\ref{sec:appendix_basis});
for the Chebyshev and Hermite branches, $\alpha_{t_i}^{(b)}$ is obtained via MLS, where given a basis family $b$ with feature map
$\phi^{(b)}$ the MLS solver fits the cached history and returns a
branch-specific coefficient vector. Other solvers are equally
compatible with this interface.
Different basis families capture different aspects of the feature trajectory, so HybridFF maintains multiple complementary branches and \emph{fuses their coefficient vectors} to obtain the final $\alpha_{t_i}$, as summarized in Fig.~\ref{fig:pipeline} and detailed in the hybrid fusion rule below.

\paragraph{MLS Solver.}
Given a basis family $b$ with a feature map $\phi^{(b)}: [0,1] \to \mathbb{R}^{m_b}$,
we model the feature trajectory as $z(\tau) \approx \phi^{(b)}(\tau)^{\!\top} c$ and find the polynomial coefficient $c_{t_i}^{(b)} \in \mathbb{R}^{K}$ that minimizes the weighted fitting error over cached history:
\begin{equation}
c_{t_i}^{(b)} = \arg\min_{c}\; \sum_{j=1}^{K} w_{t_{i,j}}^{(b)} \Big\lVert 
\phi^{(b)}(\tau_{t_i,t_{i,j}})^{\top} c - z_{t_{i,j}} \Big\rVert^2 + 
\lambda_b \lVert c \rVert^2,
\label{eq:mls-objective}
\end{equation}
where $\tau_{t_i,t_{i,j}}$ is a normalized coordinate for history step $t_{i,j}$, $w_{t_{i,j}}^{(b)} \ge 0$ is a locality weight emphasizing recent observations, and $\lambda_b$ is a ridge regularizer. The branch prediction coefficients $\alpha_{t_i}^{(b)}$ in~\eqref{eq:linear-combination} are derived from $c_{t_i}^{(b)}$ by evaluating the basis at the target step. We now define each ingredient.

\paragraph{Coordinate normalization.}
For numerical conditioning, the history and target coordinates are jointly mapped to $[0,1]$. Let $\ell_{t_i} = \min_j t_{i,j}$ and $u_{t_i} = \max(t_i, \max_j t_{i,j})$. Define $\tau_{t_i,t} = \frac{t - \ell_{t_i}}{u_{t_i} - \ell_{t_i} + \varepsilon}, \quad t = \ell_{t_i}, \dots, u_{t_i}$, with $\varepsilon > 0$ a stabilizer.

\paragraph{Locality weights.}
To emphasize recent history, each observation is weighted by a Gaussian decay on the raw step distance: $w_{t_{i,j}}^{(b)} = \exp \left(-\frac{(t_i - 
t_{i,j})^2}{(\sigma^{(b)})^2}\right)$.
The weights use raw coordinates rather than normalized $\tau$ so that the decay scale $\sigma^{(b)}$ remains tied to actual sampling-step distances.

\paragraph{Closed-form solution.} 
Define the basis matrix and weight matrix
\begin{align*}
\Phi_{t_i}^{(b)} &= \begin{bmatrix}
\phi^{(b)}(\tau_{t_i,t_{i,1}})^\top \\
\vdots \\
\phi^{(b)}(\tau_{t_i,t_{i,K}})^\top
\end{bmatrix}, \\
W_{t_i}^{(b)} &=
\mathrm{diag}\big(w_{t_{i,1}}^{(b)},\dots,w_{t_{i,K}}^{(b)}\big).
\end{align*}
Solving~\eqref{eq:mls-objective} and evaluating at the target coordinate $\tau_{t_i,t_i}$ yields the branch coefficient vector over cached features: 
\begin{align}
\alpha_{t_i}^{(b)} &= \big(\phi^{(b)}(\tau_{t_i,t_i})^{\top} \big(G_{t_i}^{(b)}\big)^{-1}\Phi_{t_i}^{(b)\top} W_{t_i}^{(b)}\big)^T, \\
G_{t_i}^{(b)} &:= \Phi_{t_i}^{(b)\top} W_{t_i}^{(b)} \Phi_{t_i}^{(b)} + \lambda_b I,
\label{eq:mls-solution}
\end{align}
The coefficients are then normalized to enforce an affine constraint,
\begin{equation}
\widetilde{\alpha}_{t_i}^{(b)} = \frac{\alpha_{t_i}^{(b)}}{\mathbf{1}^{\!\top}\alpha_{t_i}^{(b)}},
\label{eq:affine-norm}
\end{equation}
so that $\mathbf{1}^{\!\top}\widetilde{\alpha}_{t_i}^{(b)} = 1$. This preserves constants exactly and is required for the anchored residual form.

\paragraph{Basis library.}
The MLS solver in~\eqref{eq:mls-solution} is parameterized by the choice of
basis $\phi^{(b)}$, and any polynomial or orthogonal family can be plugged in.
In our experiments we instantiate three complementary branches: a
\emph{Chebyshev} branch ($\phi^{(\mathrm{Chebyshev})}$) whose orthogonality
yields well-conditioned Gram matrices, suitable for wider history windows and globally
smooth trajectories; a \emph{Hermite} branch ($\phi^{(\mathrm{Hermite})}$) with a
domain-contracting scale that prevents exponential growth of high-order terms
and is responsive to local curvature; and a \emph{Taylor} branch ($\alpha_{t_i}^{(\mathrm{Taylor})}$) whose coefficients are computed in closed form, providing strong short-range extrapolation at no matrix-solve cost.
The explicit recurrences and the Taylor weight formula are given in Sec.~\ref{sec:appendix_basis}.

\subsection{Hybrid Basis Fusion}\label{sec:method:fusion}

\paragraph{Why fuse?}
The three branches are individually optimal in different regimes: 
Taylor excels at short-range extrapolation but degrades when the target step is far from the latest cache;
Chebyshev provides numerically stable global fits but is less responsive to abrupt local change;
Hermite captures mid-range curvature but requires careful scaling.
No single branch dominates across diffusion timesteps and model layers. 
HybridFF therefore combines the branches at the coefficient level.

\paragraph{Coefficient-level fusion.}
Let $\mathcal{B} = \{\text{Chebyshev},\allowbreak \text{Hermite},\allowbreak \text{Taylor}\}$ and let $\pi_{t_i} \in \Delta^{|\mathcal{B}|-1}$ be a simplex weight vector ($\pi_{t_i}^{(b)} \ge 0$, $\sum_b \pi_{t_i}^{(b)} = 1$). The hybrid coefficient vector is
\begin{equation}
\alpha_{t_i} = \sum_{b \in \mathcal{B}} \pi_{t_i}^{(b)} \, \widetilde{\alpha}_{t_i}^{(b)}.
\label{eq:hybrid-hero}
\end{equation}
Because the per-branch coefficients are affine ($\mathbf{1}^{\!\top}\widetilde{\alpha}_{t_i}^{(b)} = 1$) and $\pi_{t_i}$ lies on the simplex, the fused coefficients automatically satisfy $\mathbf{1}^{\!\top}\alpha_{t_i} = 1$.
Let $\hat z_{t_i}^{(b)}$ be the per-branch reconstruction.
By expanding Eq.~\eqref{eq:hybrid-hero}, we recover Eq.~\eqref{eq:linear-combination}:
\begin{equation}
\hat{z}_{t_i}
= \sum_{b \in \mathcal{B}} \pi_{t_i}^{(b)}
\underbrace{\left(\sum_{j=1}^{K} \widetilde{\alpha}_{t_{i,j}}^{(b)}
z_{t_{i,j}}\right)}_{\hat z_{t_i}^{(b)}} 
= \sum_{j=1}^{K}
\underbrace{\left(\sum_{b \in \mathcal{B}} \pi_{t_i}^{(b)}
\widetilde{\alpha}_{t_{i,j}}^{(b)}\right)}_{\alpha_{t_{i,j}}}
z_{t_{i,j}} .
\label{eq:coeff-fusion}
\end{equation}

\paragraph{Anchored Residual Reconstruction}\label{sec:recon}

Let $a(t_i) = t_{i,K}$ be the most recent full step in $\mathcal{H}_{t_i}$ and $z_a := z_{a(t_i)}$ the anchor feature. HybridFF uses the coefficients in residual form:
\begin{equation}
\hat{z}_{t_i} \;=\; z_{a} \;+\; \sum_{j=1}^{K} 
\alpha_{t_{i,j}}\,\big(z_{t_{i,j}} - z_{a}\big).
\label{eq:anchored-recon}
\end{equation}
Because $\mathbf{1}^{\!\top}\alpha_{t_i} = 1$, Eq.~\eqref{eq:anchored-recon} is
algebraically identical to the affine combination
Eq.~\eqref{eq:linear-combination} ($\sum_j \alpha_{t_{i,j}} z_{t_{i,j}}$). The 
two forms differ in numerical behaviour: the residuals $z_{t_{i,j}} - z_a$ have 
substantially smaller magnitude than the raw activations, 
so~\eqref{eq:anchored-recon} better preserves numerical precision and is 
robust to large feature norms commonly observed in deep model layers (e.g., 
transformer blocks).

\subsection{Fixed and Adaptive Fusion Weights}\label{sec:method:adapt}

We report two variants under the same coefficient-fusion pipeline. They differ only in how the simplex weight $\pi_{t_i}$ is determined: HybridFF (Fixed) calibrates $\pi$ once offline and freezes it, whereas HybridFF (Adaptive) updates $\pi$ online from full-step prediction errors.

\paragraph{Fixed weights.}
The fixed weight is calibrated once by a lightweight simplex grid search over a small candidate set, scoring each candidate by SSIM against the full-step reference on a small calibration subset (a few percent of the prompts), and then frozen for all inference steps and prompts. This is a one-time hyperparameter calibration.
Because the weight is a single dimension-free simplex vector shared across steps, HybridFF (Fixed) adds essentially no runtime overhead beyond the coefficient mix itself. Per-backbone calibration details are given in Sec.~\ref{sec:appendix_ablation}.

\paragraph{Adaptive weights.}
A single fixed weight cannot track feature dynamics that vary across samples, layers, and timesteps. HybridFF (Adaptive) keeps the same branch library but recalibrates $\pi$ on the fly, reusing computation that is already available at full steps. At a full step $t_i \in \mathcal{U}$, the exact feature $z_{t_i}$ is computed by the backbone. Before inserting it into the cache, each branch makes a probe prediction $\hat z_{t_i}^{(b)}$ from the existing history using the same branch coefficient rule used in Eq.~\eqref{eq:coeff-fusion}. We measure the relative squared prediction error of branch $b$, $d_{t_i}^{(b)} =
\frac{\lVert \hat z_{t_i}^{(b)} - z_{t_i} \rVert_2^2}
{\lVert z_{t_i} \rVert_2^2 + \varepsilon},$
where $\varepsilon$ is a small stabilizer. Let $p(t_i)$ denote the previous full step carrying an adaptive state. HybridFF maintains an exponential moving average (EMA) error score per branch, $m_{t_i}^{(b)} =
\beta\, m_{p(t_i)}^{(b)} + (1-\beta)\, d_{t_i}^{(b)}, \ 0 \le \beta < 1,$
and converts these scores into adaptive weights by a softmax over negative errors, $\pi_{t_i,\mathrm{ada}}^{(b)} =
\frac{\exp(-m_{t_i}^{(b)}/\tau)}
{\sum_{c\in\mathcal{B}}\exp(-m_{t_i}^{(c)}/\tau)}, \ 
\tau > 0.$
The resulting $\pi_{t_i,\mathrm{ada}}^{(b)}$ is used as $\pi_{t_i}^{(b)}$ in Eq.~\eqref{eq:hybrid-hero} for the subsequent prediction steps until the next full-step update, so a branch with lower recent error contributes more to the fused coefficient vector. Since the exact feature $z_{t_i}$ is already available at full steps, this adds only lightweight branch probes, scalar EMA updates, and a simplex normalization, with no extra denoiser forward pass: it only alters coefficient fusion, leaving the schedule, MLS locality weights, and denoiser unchanged. Empirically, this yields consistent fidelity gains over HybridFF (Fixed) at a marginal speed cost (Tables~\ref{tab:dit_main}--\ref{tab:t2v_main}).

\subsection{Basis Library Details}\label{sec:appendix_basis}

This subsection provides the explicit definitions of the three basis branches introduced above.

\begin{enumerate}
    \item \textbf{Chebyshev branch} ($\phi^{(\mathrm{Chebyshev})}$). First-kind 
    Chebyshev polynomials on $x = (2\tau - 1)\,\gamma_c$, where $\gamma_c$ is 
    the scale parameter:
    \begin{equation}
    \begin{aligned}
    T_0(x) &= 1, \\
    T_1(x) &= x, \\
    T_n(x) &= 2x\,T_{n-1}(x) - T_{n-2}(x).
    \end{aligned}
    \label{eq:chebyshev}
    \end{equation}
    The orthogonality of $T_n$ yields well-conditioned Gram matrices, making this branch suitable for wider history windows and globally smooth trajectories.

    \item \textbf{Hermite branch} ($\phi^{(\mathrm{Hermite})}$). Scaled
    physicist's Hermite polynomials with scale parameter $\gamma > 0$,
    evaluated at $x = 2\tau - 1$:
    \begin{equation}
    \begin{aligned}
    \widetilde{H}_0(x) &= 1, \\
    \widetilde{H}_1(x) &= 2\gamma^2 x, \\
    \widetilde{H}_n(x) &= 2\gamma^2 x\,\widetilde{H}_{n-1}(x) - 2(n\!-\!1)\gamma^2\,\widetilde{H}_{n-2}(x).
    \end{aligned}
    \label{eq:hermite}
    \end{equation}
    This normalization corresponds to $\widetilde{H}_n(x)=\gamma^n H_n(\gamma x)$, where $H_n$ denotes the standard physicist's Hermite polynomial; the scale $\gamma$ contracts the effective domain, preventing exponential growth of high-order terms and making this branch responsive to local curvature.

    \item \textbf{Taylor branch} ($\alpha^{(\mathrm{Taylor})}$). 
    Finite-difference extrapolation weights computed in closed form. Let $r = 
    (t_i - t_{i,K})/(t_{i,K} - t_{i,K-1})$ be the normalised step ratio. The 
    coefficients are
    \begin{equation}
    \alpha_{t_{i,j}}^{(\mathrm{Taylor})} =
    (-1)^{j}\,\frac{r^{\,j}}{j!}\sum_{m=0}^{P-j}\frac{r^{\,m}}{m!},
    \label{eq:taylor}
    \end{equation}
    where $P$ is the extrapolation order. No matrix solve is required; only the two most recent cached features are needed, giving minimal memory cost.
\end{enumerate}

\subsection{Algorithm}\label{sec:appendix_algorithm}

Algorithm~\ref{alg:hybridff} summarizes the full HybridFF inference procedure for one sub-feature.

\begin{algorithm}[t]
  \caption{HybridFF inference for one sub-feature.}
  \label{alg:hybridff}
  \KwIn{schedule $(\mathcal{U}, \mathcal{V})$, history buffer $\mathcal{H}$,
  weights $\pi$}
  \For{$i = 1, \ldots, N$}{
    \eIf{$t_i \in \mathcal{U}$ \tcp*[f]{full step}}{
      compute $z_{t_i}$ exactly\;
      append to $\mathcal{H}$\;
    }{
      select window $K$\;
      compute $\widetilde{\alpha}_{t_i}^{(b)}$ for each $b\in\mathcal{B}$\;
      fuse $\alpha_{t_i}=\sum_b\pi_{t_i}^{(b)}
      \widetilde{\alpha}_{t_i}^{(b)}$\;
      reconstruct $\hat{z}_{t_i}$ via anchored residuals in
      Eq.~\eqref{eq:anchored-recon}\;
    }
  }
\end{algorithm}

\subsection{Scheduling and Complexity Details}\label{sec:appendix_schedule_complexity}

This subsection expands on the scheduling and complexity discussion.

\paragraph{Schedule.}
HybridFF separates the scheduler from the predictor. We support two policies:
(i) \emph{fixed-interval} scheduling with an initial full-step warm-up and an
optional full-step tail, and (ii) \emph{flex-window} scheduling in which the
inter-full-step interval grows monotonically through sampling, mirroring the budget
allocation pattern observed in~\cite{han2026spectrum}. Both policies feed an
ordered $\mathcal{I}_{t_i}$ to the predictor; no part of the predictor depends
on the schedule type.

\paragraph{Complexity.} Per predict step and per layer for DiT, HybridFF performs at
most two $K\times K$ regularized solves (Chebyshev, Hermite) in
$\mathrm{float64}$, one closed-form Taylor weight computation, a
$K$-dimensional simplex mixing, and a $K$-term anchored reconstruction over the
$d$-dimensional sub-feature. The total cost is $O(K^3 + Kd)$ with $K \le 8$ in
all our experiments, dominated by the $O(Kd)$ reconstruction term. For
comparison, a DiT block computes attention ($O(N^2d)$ where $N$ is the token
count) and MLP ($O(Nd^2)$).
Since this computation only involves small history 
windows and low-dimensional coefficient fitting, its runtime overhead is 
negligible compared with a full DiT block evaluation.

\subsection{Theoretical Properties}\label{sec:appendix_theory_properties}

We provide several simple properties of the proposed coefficient-based
hybrid predictor. They clarify why coefficient fusion is well-defined and why
the prediction remains stable under standard conditioning assumptions.

\begin{theorem}[{Stability of MLS Coefficient Prediction}]
	\label{thm:coefficient_stability}
	{Fix a predict step $t_i$. Under standard MLS positive-definiteness
	assumptions on the Gram matrix $G_{t_i}^{(b)}$
	(the MLS solution in Eq.~\eqref{eq:mls-solution}) and a bounded affine-normalization
	denominator $|\mathbf{1}^{\top}\alpha_{t_i}^{(b)}|\ge\gamma_b>0$, the
	hybrid reconstruction $\hat z_{t_i}$ is Lipschitz continuous in the cached
	features $\{z_{t_{i,j}}\}_{j=1}^{K}$ with Lipschitz constant $\lVert\alpha_{t_i}\rVert_1$, which is further bounded by}
	\[
	\lVert\alpha_{t_i}\rVert_1
	\;\le\;
	\sum_{b\in\mathcal{B}}\pi_{t_i}^{(b)}\,\lVert\widetilde\alpha_{t_i}^{(b)}\rVert_1
	.
	\]
\end{theorem}

\begin{proof}
\textbf{Step 1: per-branch coefficient norm.}
From the MLS solution in Eq.~\eqref{eq:mls-solution},
\[
\alpha_{t_i}^{(b)} \;=\; \phi^{(b)}(\tau_{t_i,t_i})^{\!\top} \,\bigl(G_{t_i}^{(b)}\bigr)^{-1}\, \Phi_{t_i}^{(b)\top}\, W_{t_i}^{(b)} .
\]
By assumption, $G_{t_i}^{(b)}$ is positive definite with $\lambda_{\min}(G_{t_i}^{(b)})\ge\mu_b>0$, so $\bigl\lVert (G_{t_i}^{(b)})^{-1} \bigr\rVert_2 \le 1/\mu_b$. Using $\lVert v\rVert_1\le\sqrt{K}\,\lVert v\rVert_2$ for any $v \in \mathbb{R}^{K}$ and the sub-multiplicativity of the spectral norm,
\begin{align*}
\bigl\lVert \alpha_{t_i}^{(b)} \bigr\rVert_1
&\le \sqrt{K} \bigl\lVert \alpha_{t_i}^{(b)} \bigr\rVert_2 \\
&\le \frac{\sqrt{K}}{\mu_b}
\bigl\lVert \phi^{(b)}(\tau_{t_i,t_i}) \bigr\rVert_2
\bigl\lVert \Phi_{t_i}^{(b)\top} W_{t_i}^{(b)} \bigr\rVert_2 .
\end{align*}
After affine normalization in Eq.~\eqref{eq:affine-norm} and using $\bigl|\mathbf{1}^{\!\top}\alpha_{t_i}^{(b)}\bigr|\ge\gamma_b>0$,
\begin{equation}
\begin{aligned}
\bigl\lVert \widetilde\alpha_{t_i}^{(b)} \bigr\rVert_1
&= \frac{\bigl\lVert \alpha_{t_i}^{(b)} \bigr\rVert_1}
{\bigl|\mathbf{1}^{\!\top}\alpha_{t_i}^{(b)}\bigr|} \\
&\le \frac{\sqrt{K}}{\mu_b\,\gamma_b}
\bigl\lVert \phi^{(b)}(\tau_{t_i,t_i}) \bigr\rVert_2
\bigl\lVert \Phi_{t_i}^{(b)\top} W_{t_i}^{(b)} \bigr\rVert_2 .
\end{aligned}
\label{eq:appendix_branch_bound}
\end{equation}
\textbf{Step 2: hybrid coefficient norm.}
Since $\pi_{t_i}\in\Delta^{|\mathcal{B}|-1}$, the triangle inequality gives
\begin{equation}
\bigl\lVert \alpha_{t_i} \bigr\rVert_1
=
\Bigl\lVert \textstyle\sum_{b\in\mathcal{B}} \pi_{t_i}^{(b)}\, \widetilde\alpha_{t_i}^{(b)} \Bigr\rVert_1
\le
\sum_{b\in\mathcal{B}} \pi_{t_i}^{(b)}\, \bigl\lVert \widetilde\alpha_{t_i}^{(b)} \bigr\rVert_1 .
\label{eq:appendix_hybrid_bound}
\end{equation}
\textbf{Step 3: Lipschitz continuity in the cached features.}
The hybrid reconstruction $\hat z_{t_i}=\sum_{j=1}^{K} \alpha_{t_{i,j}}\, z_{t_{i,j}}$ is a fixed linear functional of $(z_{t_{i,1}},\dots,z_{t_{i,K}})$. Perturb each cached feature by $\delta_j\in\mathbb{R}^{d}$ to obtain a perturbed history $z_{t_{i,j}}+\delta_j$, with corresponding reconstruction $\hat z_{t_i}+\sum_j\alpha_{t_{i,j}}\,\delta_j$. Therefore
\[
  \begin{aligned}
  \Bigl\lVert \hat z_{t_i}^{\,\mathrm{pert}} - \hat z_{t_i} \Bigr\rVert_2
  &= \Bigl\lVert \textstyle\sum_{j=1}^{K} \alpha_{t_{i,j}}\, \delta_j \Bigr\rVert_2 \\
  &\le \sum_{j=1}^{K} |\alpha_{t_{i,j}}|\, \lVert \delta_j \rVert_2 \\
  &\le \bigl\lVert \alpha_{t_i} \bigr\rVert_1 \cdot \max_{1\le j\le K} \lVert \delta_j \rVert_2 .
  \end{aligned}
  \]
where the last inequality uses $\sum_j |\alpha_{t_{i,j}}|\,\lVert\delta_j\rVert_2 \le (\sum_j |\alpha_{t_{i,j}}|)\cdot\max_j\lVert\delta_j\rVert_2 = \lVert\alpha_{t_i}\rVert_1\cdot\max_j\lVert\delta_j\rVert_2$. Combining with~\eqref{eq:appendix_hybrid_bound} and~\eqref{eq:appendix_branch_bound} yields the Lipschitz constant claimed in Theorem~\ref{thm:coefficient_stability}. 
\end{proof}

\begin{theorem}[{Hybrid Approximation Bound}]
	\label{thm:hybrid_approximation}
	{Let $e_{t_i}^{(b)} = \hat z_{t_i}^{(b)} - z_{t_i}$ denote the
	per-branch prediction error against the true feature $z_{t_i}$, with each
	branch prediction $\hat z_{t_i}^{(b)}$ defined in
	Eq.~\eqref{eq:coeff-fusion}. Then the
	best simplex fusion is no worse than the best single branch:}
	\[
	\min_{\pi_{t_i}\in\Delta^{|\mathcal{B}|-1}}
	\bigl\lVert
	\textstyle\sum_{b\in\mathcal{B}}\pi_{t_i}^{(b)}\hat z_{t_i}^{(b)}
	-z_{t_i}
	\bigr\rVert_2
	\;\le\;
	\min_{b\in\mathcal{B}}
	\lVert\hat z_{t_i}^{(b)}-z_{t_i}\rVert_2 .
	\]
\end{theorem}

\begin{proof}
\textbf{Step 1: linearity of the prediction error.}
By construction, $\hat z_{t_i} = \sum_{b\in\mathcal{B}} \pi_{t_i}^{(b)}\, \hat z_{t_i}^{(b)}$ and $\sum_{b\in\mathcal{B}} \pi_{t_i}^{(b)} = 1$. Subtracting $z_{t_i}$,
 \[
  \begin{aligned}
  \hat z_{t_i} - z_{t_i}
  &= \sum_{b\in\mathcal{B}} \pi_{t_i}^{(b)}\, \hat z_{t_i}^{(b)} - z_{t_i} \\
  &= \sum_{b\in\mathcal{B}} \pi_{t_i}^{(b)}\, \bigl( \hat z_{t_i}^{(b)} - z_{t_i} \bigr) \\
  &= \sum_{b\in\mathcal{B}} \pi_{t_i}^{(b)}\, e_{t_i}^{(b)} .
  \end{aligned}
  \]
The triangle inequality then gives
\begin{equation}
\bigl\lVert \hat z_{t_i} - z_{t_i} \bigr\rVert_2
\le
\sum_{b\in\mathcal{B}} \pi_{t_i}^{(b)}\, \bigl\lVert e_{t_i}^{(b)} \bigr\rVert_2 .
\label{eq:appendix_convex_combination}
\end{equation}
\textbf{Step 2: optimal simplex fusion dominates the best single branch.}
Let $b^{\star} \in \arg\min_{b\in\mathcal{B}} \lVert e_{t_i}^{(b)}\rVert_2$, and consider the corner $\bar{\pi}\in\Delta^{|\mathcal{B}|-1}$ defined by $\bar{\pi}^{(b^{\star})} = 1$ and $\bar{\pi}^{(b)} = 0$ for $b \neq b^{\star}$. Substituting $\bar{\pi}$ into the left-hand side of~\eqref{eq:appendix_convex_combination},
\begin{align*}
&\Bigl\lVert \textstyle\sum_{b\in\mathcal{B}} \bar{\pi}^{(b)}\,
\hat z_{t_i}^{(b)} - z_{t_i} \Bigr\rVert_2 \\
&\qquad= \bigl\lVert \hat z_{t_i}^{(b^{\star})} - z_{t_i} \bigr\rVert_2
= \min_{b\in\mathcal{B}}
\bigl\lVert \hat z_{t_i}^{(b)} - z_{t_i} \bigr\rVert_2 .
\end{align*}
Since the minimum over $\Delta^{|\mathcal{B}|-1}$ is no larger than the value at any particular feasible point $\bar{\pi}$,
\[
\min_{\pi_{t_i}\in\Delta^{|\mathcal{B}|-1}}
\Bigl\lVert \textstyle\sum_{b\in\mathcal{B}} \pi_{t_i}^{(b)}\, \hat z_{t_i}^{(b)} - z_{t_i} \Bigr\rVert_2
\le
\min_{b\in\mathcal{B}} \bigl\lVert \hat z_{t_i}^{(b)} - z_{t_i} \bigr\rVert_2 ,
\]
which is the statement of Theorem~\ref{thm:hybrid_approximation}.

\end{proof}

\begin{remark}
    Inequality~\eqref{eq:appendix_convex_combination} is tight only when all branches share the same error direction (in which case the upper bound is attained by every $\pi$). In practice the per-branch errors $e_{t_i}^{(b)}$ point in different directions, and partial cancellation yields strictly smaller hybrid error than the convex bound, which empirically corroborates the observations in the main experiments.
    
	Theorem~\ref{thm:hybrid_approximation} explains the role of the hybrid
	basis bank. Taylor, Hermite, and Chebyshev branches induce different
	coefficient operators and therefore different local extrapolation errors.
	Simplex fusion allows the predictor to inherit the favorable branch
	behavior at each step, while Theorem~\ref{thm:coefficient_stability} shows
	that this fusion does not introduce additional instability beyond the
	weighted coefficient norms of the active branches.
\end{remark}

%% file: src_combination/results.tex
\section{Experiments}
\label{sec:experiments}

\subsection{Experimental Setup}\label{sec:appendix_setup}

\paragraph{Models and protocols.}
We evaluate HybridFF on class-conditional image generation with DiT-XL/2~\cite{peebles2023scalable}, text-to-image generation with FLUX.1-dev~\cite{esser2024flux} and SD3.5-Large~\cite{esser2024}, and text-to-video generation with HunyuanVideo. We use the standard ImageNet protocol for DiT-XL/2, DrawBench prompts for FLUX.1-dev and SD3.5-Large, and VBench prompts for HunyuanVideo. The 50-step sampler serves as the quality anchor on every benchmark, and all experiments are conducted on NVIDIA RTX PRO 6000 GPUs.

\paragraph{Baselines and hyperparameters.}
We compare against three representative forecasting baselines: TaylorSeer~\cite{liu2025taylorseer} (Taylor extrapolation), HiCache~\cite{feng2026hicache} (Hermite forecasting), and Spectrum~\cite{han2026spectrum} (Chebyshev forecasting). Unless otherwise specified, all forecasting baselines use order $O=4$ or degree $D=4$, and HybridFF additionally sets the history window to $K=5$.

\paragraph{Metrics.}
We report both efficiency metrics (wall-clock latency and speedup w.r.t.\ the 50-step reference) and quality metrics (SSIM, PSNR, LPIPS w.r.t.\ the 50-step output; FID, IS, sFID, Recall, and Precision for class-conditional generation; CLIP-Score and ImageReward for text-to-image; and VBench quality for text-to-video).

\paragraph{Note on latency.}
HybridFF is consistently slightly faster than Spectrum across all benchmarks even though it evaluates three branches. Spectrum performs its Chebyshev fit over the \emph{entire} cached history, whereas HybridFF restricts every MLS branch to a short window ($K\!=\!5$ in all experiments) and uses a closed-form Taylor branch with no matrix solve. The cubic $O(K^3)$ MLS solve is therefore much smaller, and this gain outweighs the cost of the extra branches and the simplex mix.

\subsection{Class-Conditional Image Generation}

We first evaluate on DiT-XL/2 with fixed HybridFF weights $\pi=(0.9, 0.1, 0.0)$ and ridge regularizer $\lambda_b=10^{-4}$. Table~\ref{tab:dit_main} reports results across three acceleration regimes.
\begin{table*}[!ht]
\centering
\setlength{\tabcolsep}{2pt}
\renewcommand{\arraystretch}{1.08}
\begin{tabular}{clcccccccccc}
\toprule
NFE & Method & SSIM$\uparrow$ & PSNR$\uparrow$ & LPIPS$\downarrow$ & FID$\downarrow$ & IS$\uparrow$ & sFID$\downarrow$ & Recall$\uparrow$ & Precision$\uparrow$ & Latency (s)$\downarrow$ & Speedup$\uparrow$ \\
\midrule
50 & Baseline (50-step) & \resultsNA & \resultsNA & \resultsNA & 2.276 & 240.855 & 4.258 & 0.592 & 0.802 & 0.971 & 1.00$\times$ \\
\midrule
14 & TaylorSeer & 0.927 & 30.710 & 0.042 & 2.528 & 235.309 & 5.121 & 0.589 & \textbf{0.803} & 0.568 & 1.71$\times$ \\
14 & HiCache    & 0.927 & 30.811 & 0.041 & 2.518 & 235.894 & 5.165 & 0.592 & \underline{0.801} & 0.531 & 1.83$\times$ \\
14 & Spectrum   & 0.914 & 29.233 & 0.052 & 2.604 & 230.968 & \textbf{4.577} & \textbf{0.596} & 0.794 & 0.542 & 1.79$\times$ \\
14 & \textbf{HybridFF (Fixed)} & \underline{0.942} & \underline{31.982} & \underline{0.032} & \underline{2.467} & \underline{236.325} & 4.838 & \underline{0.593} & 0.800 & \textbf{0.520} & \textbf{1.87$\times$} \\
14 & \textbf{HybridFF (Adaptive)} & \textbf{0.945} & \textbf{32.053} & \textbf{0.031} & \textbf{2.416} & \textbf{236.801} & \underline{4.762} & 0.591 & \underline{0.801} & \underline{0.527} & \underline{1.84$\times$} \\
\midrule
10 & TaylorSeer & 0.825 & 25.460 & 0.106 & 3.113 & 225.260 & 6.312 & 0.579 & \underline{0.788} & 0.481 & 2.02$\times$ \\
10 & HiCache    & 0.830 & 25.543 & 0.100 & 3.087 & 225.682 & 6.835 & 0.583 & \textbf{0.789} & 0.447 & 2.17$\times$ \\
10 & Spectrum   & 0.821 & 24.881 & 0.115 & 3.313 & 219.220 & \textbf{4.758} & \underline{0.597} & 0.774 & 0.454 & 2.14$\times$ \\
10 & \textbf{HybridFF (Fixed)} & \underline{0.843} & \underline{25.926} & \underline{0.092} & \underline{2.955} & \underline{226.384} & 5.489 & \textbf{0.598} & 0.784 & \textbf{0.437} & \textbf{2.22$\times$} \\
10 & \textbf{HybridFF (Adaptive)} & \textbf{0.848} & \textbf{26.114} & \textbf{0.091} & \textbf{2.912} & \textbf{226.573} & \underline{4.924} & 0.594 & 0.783 & \underline{0.445} & \underline{2.18$\times$} \\
\midrule
 8 & TaylorSeer & 0.721 & 22.339 & 0.181 & 4.465 & 204.265 & 7.322 & 0.568 & \textbf{0.765} & 0.432 & 2.25$\times$ \\
 8 & HiCache    & 0.722 & 22.319 & 0.171 & 4.324 & 204.442 & 7.206 & 0.569 & \underline{0.762} & 0.408 & 2.38$\times$ \\
 8 & Spectrum   & 0.653 & 20.835 & 0.253 & 7.254 & 181.102 & 7.980 & 0.574 & 0.711 & 0.424 & 2.29$\times$ \\
 8 & \textbf{HybridFF (Fixed)} & \underline{0.735} & \underline{22.483} & \underline{0.165} & \underline{4.275} & \underline{205.619} & \underline{6.001} & \textbf{0.598} & 0.754 & \textbf{0.398} & \textbf{2.44$\times$} \\
 8 & \textbf{HybridFF (Adaptive)} & \textbf{0.756} & \textbf{23.251} & \textbf{0.156} & \textbf{4.263} & \textbf{206.474} & \textbf{4.924} & \underline{0.596} & 0.741 & \underline{0.403} & \underline{2.41$\times$} \\
\bottomrule
\end{tabular}
\caption{Comparison on class-conditional image generation with DiT-XL/2.
Bold and underline mark the best and second-best entries in each block, respectively.}\label{tab:dit_main}
\end{table*}
%
HybridFF (Fixed) is the fastest HybridFF variant in each block, reaching up to $2.44\times$ speedup at $\text{NFE}\!=\!8$.
HybridFF (Adaptive), which attains the best SSIM, PSNR, LPIPS, FID, and IS at every budget, gives the strongest fidelity: it improves over Fixed on these primary fidelity metrics at all three cache budgets, and at the aggressive $\text{NFE}\!=\!8$ setting it raises SSIM from $0.735$ to $0.756$, PSNR from $22.483$ to $23.251$, reduces LPIPS from $0.165$ to $0.156$, and lowers FID from $4.275$ to $4.263$, while still keeping a $2.41\times$ speedup.
Fig.~\ref{fig:dit_qual} confirms HybridFF stays visually closest to the 50-step reference.


\begin{figure}[!ht]
\centering
\includegraphics[width=0.90\linewidth]{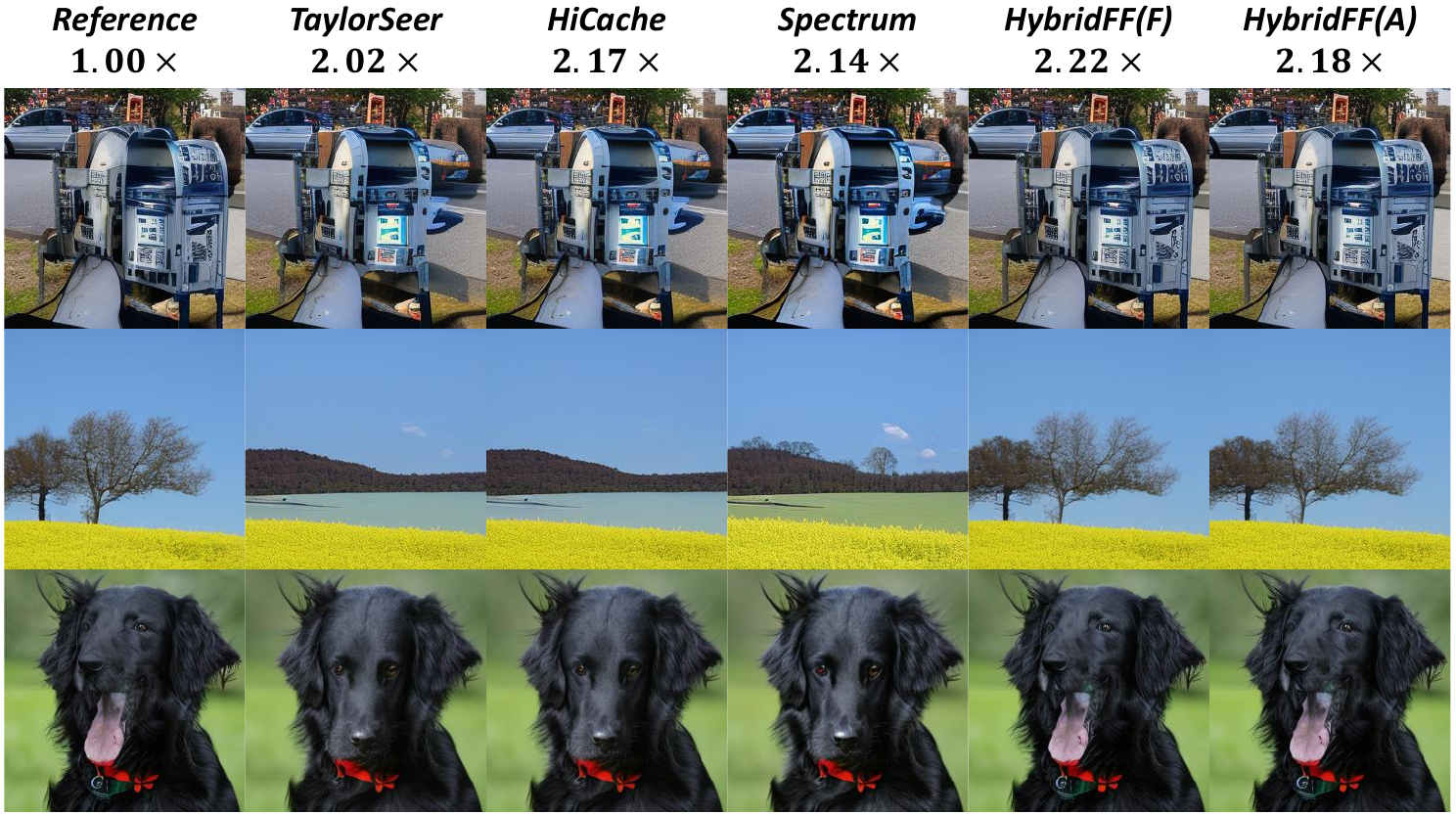}
\caption{
Comparison on class-conditional generation with DiT-XL/2 at NFE\,$=$\,10.
F/A denote fixed/adaptive fusion weights.
}
\label{fig:dit_qual}
\end{figure}

\begin{figure}[t]
\centering
\includegraphics[width=0.94\linewidth]{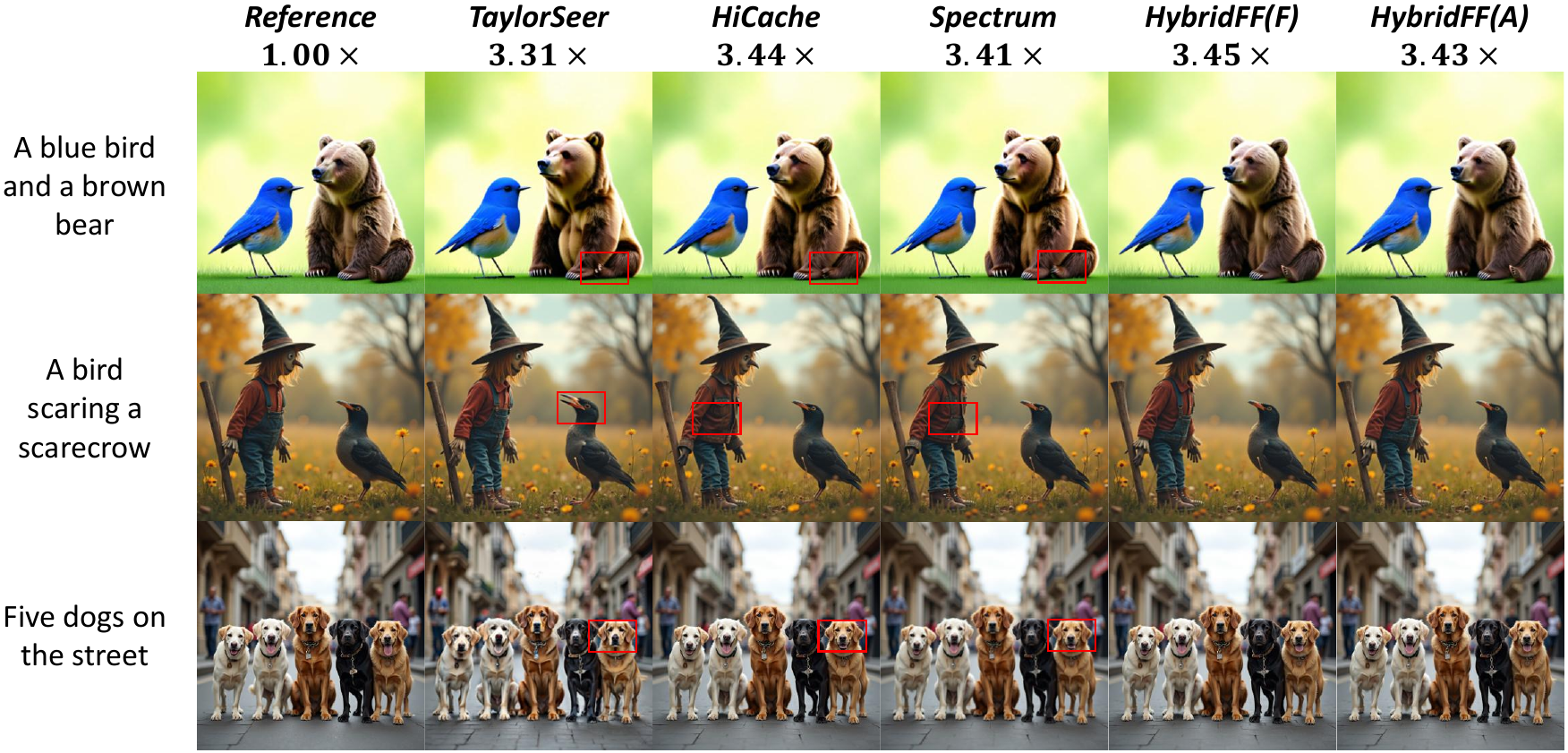}
\caption{
Comparison on text-to-image generation with SD3.5-Large at NFE\,$=$\,14.
F/A denote fixed/adaptive fusion weights.
}
\label{fig:sd3_qual}
\end{figure}

\subsection{Text-to-Image Generation}

We evaluate on FLUX.1-dev with fixed fusion weights $\pi=(0.6, 0.1, 0.3)$ and $\lambda_b=10^{-4}$, and on SD3.5-Large with $\pi=(0.1, 0.9, 0.0)$ and $\lambda_b=10^{-3}$. 
\begin{table*}[!ht]
\centering
\setlength{\tabcolsep}{6pt}
\renewcommand{\arraystretch}{1.08}
\textbf{FLUX.1-dev}\\[3pt]
\begin{tabular}{clcccccc}
\toprule
NFE & Method & SSIM$\uparrow$ & PSNR$\uparrow$ & LPIPS$\downarrow$ & IR$\uparrow$ & CS$\uparrow$ & Spd$\uparrow$ \\
\midrule
50 & Baseline (50-step) & \resultsNA & \resultsNA & \resultsNA & 0.993 & 27.504 & 1.00$\times$ \\
\midrule
14 & TaylorSeer & 0.883 & 25.204 & 0.108 & \textbf{1.017} & 27.462 & 3.13$\times$ \\
14 & HiCache & 0.893 & 26.209 & 0.098 & \underline{1.000} & 27.465 & \underline{3.24$\times$} \\
14 & Spectrum & 0.895 & 26.375 & 0.095 & 0.997 & 27.456 & 3.22$\times$ \\
14 & \textbf{HybridFF (Fixed)} & \underline{0.897} & \underline{26.411} & \underline{0.092} & 0.997 & \textbf{27.493} & \textbf{3.25$\times$} \\
14 & \textbf{HybridFF (Adaptive)} & \textbf{0.902} & \textbf{26.993} & \textbf{0.088} & 0.999 & \underline{27.492} & 3.23$\times$ \\
\midrule
10 & TaylorSeer & 0.749 & 20.133 & 0.260 & 0.962 & 27.472 & 4.31$\times$ \\
10 & HiCache & 0.820 & 22.975 & 0.186 & 0.982 & 27.495 & \underline{4.36$\times$} \\
10 & Spectrum & 0.821 & 23.031 & 0.184 & 0.983 & 27.535 & 4.33$\times$ \\
10 & \textbf{HybridFF (Fixed)} & \underline{0.822} & \underline{23.086} & \underline{0.175} & \textbf{0.985} & \underline{27.551} & \textbf{4.37$\times$} \\
10 & \textbf{HybridFF (Adaptive)} & \textbf{0.824} & \textbf{23.327} & \textbf{0.171} & \underline{0.984} & \textbf{27.556} & 4.35$\times$ \\
\bottomrule
\end{tabular}

\medskip
\textbf{SD3.5-Large}\\[3pt]
\begin{tabular}{clcccccc}
\toprule
NFE & Method & SSIM$\uparrow$ & PSNR$\uparrow$ & LPIPS$\downarrow$ & IR$\uparrow$ & CS$\uparrow$ & Spd$\uparrow$ \\
\midrule
50 & Baseline (50-step) & \resultsNA & \resultsNA & \resultsNA & 1.003 & 28.533 & 1.00$\times$ \\
\midrule
14 & TaylorSeer & 0.722 & 16.888 & 0.271 & \underline{0.986} & 28.480 & 3.31$\times$ \\
14 & HiCache & 0.755 & 17.996 & \underline{0.240} & 0.967 & 28.464 & \underline{3.44$\times$} \\
14 & Spectrum & 0.749 & 17.819 & 0.244 & \textbf{0.997} & 28.487 & 3.41$\times$ \\
14 & \textbf{HybridFF (Fixed)} & \underline{0.765} & \underline{18.550} & \textbf{0.236} & 0.965 & \underline{28.574} & \textbf{3.45$\times$} \\
14 & \textbf{HybridFF (Adaptive)} & \textbf{0.774} & \textbf{18.755} & 0.242 & 0.981 & \textbf{28.628} & 3.43$\times$ \\
\midrule
10 & TaylorSeer & 0.580 & 13.346 & 0.432 & 0.667 & 27.747 & 4.56$\times$ \\
10 & HiCache & 0.657 & 15.639 & 0.353 & 0.844 & 28.124 & \underline{4.74$\times$} \\
10 & Spectrum & 0.655 & 15.529 & 0.356 & 0.843 & 28.125 & 4.71$\times$ \\
10 & \textbf{HybridFF (Fixed)} & \underline{0.676} & \underline{16.342} & \underline{0.339} & \underline{0.889} & \underline{28.173} & \textbf{4.75$\times$} \\
10 & \textbf{HybridFF (Adaptive)} & \textbf{0.691} & \textbf{16.772} & \textbf{0.314} & \textbf{0.908} & \textbf{28.280} & 4.72$\times$ \\
\bottomrule
\end{tabular}
\caption{Quantitative comparison on text-to-image generation with FLUX.1-dev and SD3.5-Large. Bold and underline mark the best and second-best entries in each block, respectively. IR\,=\,ImageReward, CS\,=\,CLIP-Score, Spd\,=\,Speedup.}\label{tab:t2i_combined}
\end{table*}
%
HybridFF (Fixed) is the most efficient HybridFF variant on both text-to-image backbones, reaching up to $4.37\times$ speedup on FLUX.1-dev and $4.75\times$ on SD3.5-Large.
HybridFF (Adaptive) gives stronger reconstruction fidelity under aggressive caching: at $\text{NFE}\!=\!10$, it improves FLUX.1-dev PSNR from $23.086$ to $23.327$ and LPIPS from $0.175$ to $0.171$, while on SD3.5-Large it raises SSIM from $0.676$ to $0.691$, PSNR from $16.342$ to $16.772$, and reduces LPIPS from $0.339$ to $0.314$, with only a small speedup reduction from $4.75\times$ to $4.72\times$.
Fig.~\ref{fig:sd3_qual} confirms the qualitative gains.

\begin{table*}[t]
\centering
\setlength{\tabcolsep}{4.5pt}
\renewcommand{\arraystretch}{1.08}
\begin{tabular}{clcccccc}
\toprule
NFE & Method & SSIM$\uparrow$ & PSNR$\uparrow$ & LPIPS$\downarrow$ & VBench Quality$\uparrow$ & Latency (s)$\downarrow$ & Speedup$\uparrow$ \\
\midrule
50 & Baseline (50-step) & \resultsNA & \resultsNA & \resultsNA & 80.958 & 351.227 & 1.00$\times$ \\
\midrule
14 & TaylorSeer & 0.794 & 24.568 & 0.182 & 79.576 & 140.757 & 2.50$\times$ \\
14 & HiCache    & 0.826 & 24.984 & 0.165 & 80.292 & 138.094 & 2.54$\times$ \\
14 & Spectrum   & 0.834 & 26.237 & 0.140 & 80.259 & 128.548 & 2.73$\times$ \\
14 & \textbf{HybridFF (Fixed)} & \underline{0.853} & \underline{26.400} & \underline{0.124} & \underline{80.343} & \textbf{123.702} & \textbf{2.84$\times$} \\
14 & \textbf{HybridFF (Adaptive)} & \textbf{0.867} & \textbf{26.809} & \textbf{0.113} & \textbf{80.614} & \underline{126.076} & \underline{2.79$\times$} \\
\midrule
10 & TaylorSeer & 0.754 & 20.715 & 0.248 & 79.065 & 103.409 & 3.40$\times$ \\
10 & HiCache    & 0.775 & 23.046 & 0.201 & 79.946 & 100.589 & 3.50$\times$ \\
10 & Spectrum   & 0.788 & 24.473 & 0.189 & 79.895 & 95.223 & 3.69$\times$ \\
10 & \textbf{HybridFF (Fixed)} & \underline{0.801} & \underline{24.875} & \underline{0.177} & \underline{80.041} & \textbf{91.132} & \textbf{3.85$\times$} \\
10 & \textbf{HybridFF (Adaptive)} & \textbf{0.819} & \textbf{24.942} & \textbf{0.171} & \textbf{80.127} & \underline{94.681} & \underline{3.71$\times$} \\
\bottomrule
\end{tabular}
\caption{
Comparison on text-to-video generation with HunyuanVideo. 
Bold and underline mark the best and second-best entries in each block, respectively.
}\label{tab:t2v_main}
\end{table*}

\begin{figure}[t]
\centering
\includegraphics[width=0.99\linewidth]{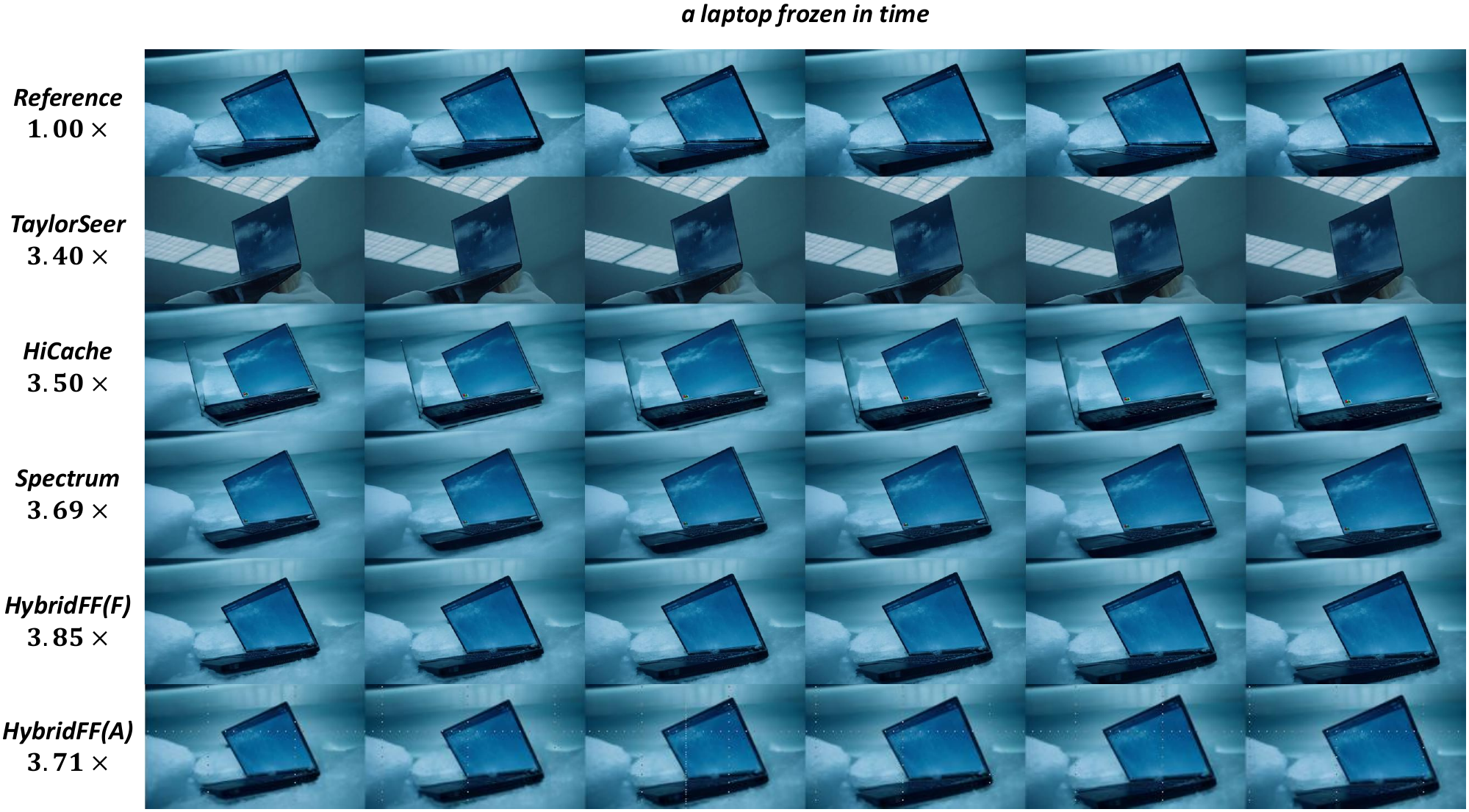}
\caption{
Comparison on text-to-video generation with HunyuanVideo at NFE\,$=$\,10.
F/A denote fixed/adaptive fusion weights.
}
\label{fig:video_qual_layout}
\end{figure}

\pagebreak[3]
\subsection{Text-to-Video Generation}

For HunyuanVideo we use fixed HybridFF weights $\pi=(0.7, 0.2, 0.1)$ and $K=5$. Each method generates $946 \times 4 = 3{,}784$ clips on the VBench prompts.
HybridFF (Fixed) is the fastest HybridFF variant in each block, cutting per-clip latency from $351.2$\,s to $91.1$\,s and reaching $3.85\times$ speedup at $\text{NFE}\!=\!10$.
HybridFF (Adaptive) gives the strongest video fidelity: at $\text{NFE}\!=\!14$ it raises SSIM from $0.853$ to $0.867$, PSNR from $26.400$ to $26.809$, and VBench Quality from $80.343$ to $80.614$ over Fixed; at $\text{NFE}\!=\!10$ it further improves SSIM to $0.819$ and LPIPS to $0.171$, while still keeping a $3.71\times$ speedup.
Fig.~\ref{fig:video_qual_layout} shows that HybridFF stays closest to the 50-step reference across frames.

\begin{figure}[h]
\centering
\includegraphics[width=0.92\linewidth]{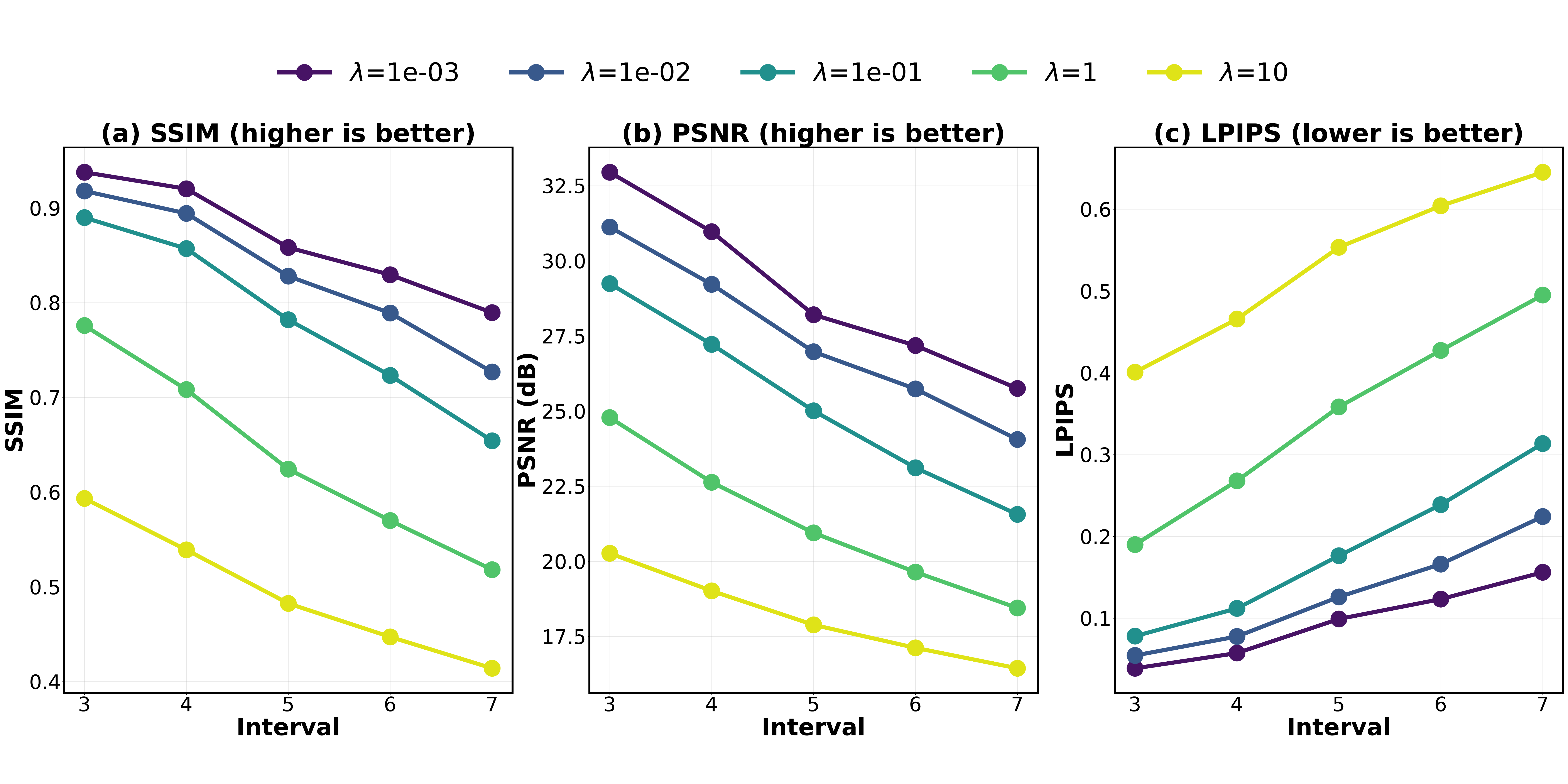}\\
{\small (a) Ridge regularizer $\lambda$.}\\[3pt]
\includegraphics[width=0.92\linewidth]{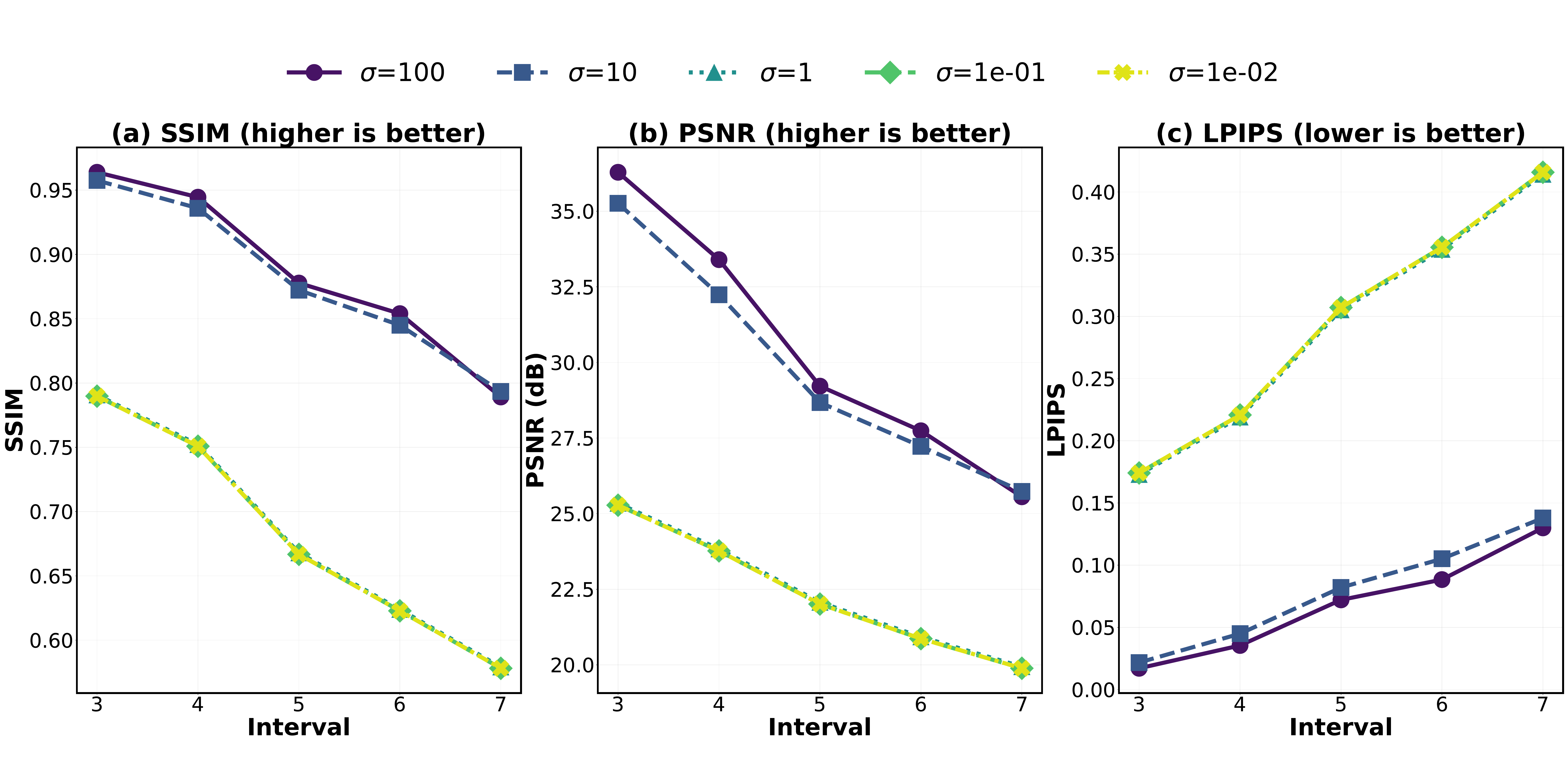}\\[-1pt]
{\small (b) Locality bandwidth $\sigma$.}
\caption{Ablations on the ridge regularizer and locality bandwidth with DiT-XL/2 at NFE\,$=$\,10.}
\label{fig:ablation_lambda_sigma}
\end{figure}

\subsection{Detailed Ablation Studies}\label{sec:appendix_ablation}

This subsection reports the complete ablation sweeps.

\paragraph{Effect of $\lambda$ and $\sigma$.}
Figure~\ref{fig:ablation_lambda_sigma} sweeps the ridge regularizer $\lambda_b$ and locality bandwidth $\sigma^{(b)}$ over their typical operating ranges. Quality is generally robust to a wide range of $\lambda$ values, with a clear plateau where SSIM is maximized; very small $\lambda$ induces ill-conditioning in the MLS solve, while very large $\lambda$ over-shrinks the coefficients toward the uniform average and degrades fidelity. Similarly, $\sigma$ controls how aggressively recent observations are emphasized: too-small $\sigma$ effectively reduces to a one-step extrapolator, while too-large $\sigma$ loses locality. The optima of both curves correspond to the default setting used in our main experiments.


\paragraph{Effect of MLS degree and window.}
Table~\ref{tab:ablation_dw} reports a grid sweep over valid polynomial degree--MLS window configurations on DiT-XL/2 at NFE $=10$. Increasing $D$ from $1$ to $4$ consistently improves all quality metrics (e.g., SSIM rises from $0.636$ at $D1W4$ to $0.843$ at $D4W5$), confirming that smaller basis libraries are too rigid for the heterogeneous feature dynamics observed in diffusion models. Within each degree, the optimum window size is small ($W\!=\!4$--$5$); larger windows mildly degrade quality because they pull in features from too-distant timesteps. Our default configuration $D\!=\!4, W\!=\!5$ achieves the best PSNR/LPIPS and near-best SSIM while also being the fastest among the high-quality configurations.

\begin{table*}[t]
\caption{Ablation on the MLS polynomial degree $D$ and history window $W$ with DiT-XL/2 at NFE\,$=$\,10.}\label{tab:ablation_dw}
\centering
\setlength{\tabcolsep}{2.0pt}
\renewcommand{\arraystretch}{1.08}
\begin{tabular}{lcccccccc}
\toprule
Config & SSIM$\uparrow$ & PSNR$\uparrow$ & LPIPS$\downarrow$ & FID$\downarrow$ & IS$\uparrow$ & sFID$\downarrow$ & Latency (s)$\downarrow$ & Speedup$\uparrow$ \\
\midrule
D1W4 & 0.636 & 21.117 & 0.283 & 9.449  & 165.282 & 16.268 & 0.410 & 2.37$\times$ \\
D1W5 & 0.579 & 20.179 & 0.348 & 16.084 & 125.705 & 23.489 & 0.419 & 2.32$\times$ \\
D1W6 & 0.538 & 19.496 & 0.400 & 24.096 & 95.156  & 31.739 & 0.426 & 2.28$\times$ \\
D1W7 & 0.513 & 19.019 & 0.437 & 31.397 & 76.817  & 39.063 & 0.433 & 2.24$\times$ \\
D1W8 & 0.499 & 18.752 & 0.457 & 36.362 & 67.100  & 43.727 & 0.442 & 2.20$\times$ \\
\midrule
D2W4 & 0.800 & 24.029 & 0.129 & 3.179 & 224.385 & 6.757  & 0.414 & 2.35$\times$ \\
D2W5 & 0.769 & 23.262 & 0.155 & 3.642 & 219.054 & 8.052  & 0.423 & 2.30$\times$ \\
D2W6 & 0.741 & 22.831 & 0.180 & 4.507 & 208.800 & 9.849  & 0.430 & 2.26$\times$ \\
D2W7 & 0.706 & 22.076 & 0.211 & 5.705 & 196.302 & 12.099 & 0.438 & 2.22$\times$ \\
D2W8 & 0.681 & 21.662 & 0.231 & 6.814 & 185.535 & 13.810 & 0.449 & 2.16$\times$ \\
\midrule
D3W4 & 0.821 & 25.014 & 0.104 & 3.702 & 210.091 & 5.237 & 0.419 & 2.32$\times$ \\
D3W5 & 0.835 & 25.187 & 0.101 & 3.102 & 220.682 & 5.554 & 0.429 & 2.26$\times$ \\
D3W6 & 0.823 & 24.557 & 0.109 & 3.108 & 221.640 & 6.112 & 0.435 & 2.23$\times$ \\
D3W7 & 0.809 & 23.655 & 0.120 & 3.256 & 219.145 & 6.844 & 0.443 & 2.19$\times$ \\
D3W8 & 0.795 & 22.869 & 0.132 & 3.482 & 216.192 & 7.584 & 0.450 & 2.16$\times$ \\
\midrule
D4W5 & 0.843 & 25.926 & 0.092 & 2.955 & 226.384 & 5.489 & 0.437 & 2.22$\times$ \\
D4W6 & 0.844 & 25.833 & 0.093 & 2.917 & 226.256 & 5.701 & 0.440 & 2.21$\times$ \\
D4W7 & 0.840 & 25.560 & 0.095 & 2.921 & 226.431 & 5.940 & 0.445 & 2.18$\times$ \\
D4W8 & 0.836 & 25.209 & 0.099 & 2.945 & 226.159 & 6.157 & 0.457 & 2.12$\times$ \\
\bottomrule
\end{tabular}
\end{table*}

\paragraph{Fixed Weight Selection.}
In our default instantiation, $\pi_{t_i}$ is determined by a lightweight grid search over the simplex on a small calibration set, evaluated using SSIM, PSNR, and LPIPS against a reference run (Figure~\ref{fig:ssim-heatmap}); the resulting weights are reused at inference. 
This formulation also leaves room for finer-grained weight designs when needed.

\begin{figure}[t]
\centering
\begin{overpic}[width=0.95\linewidth]{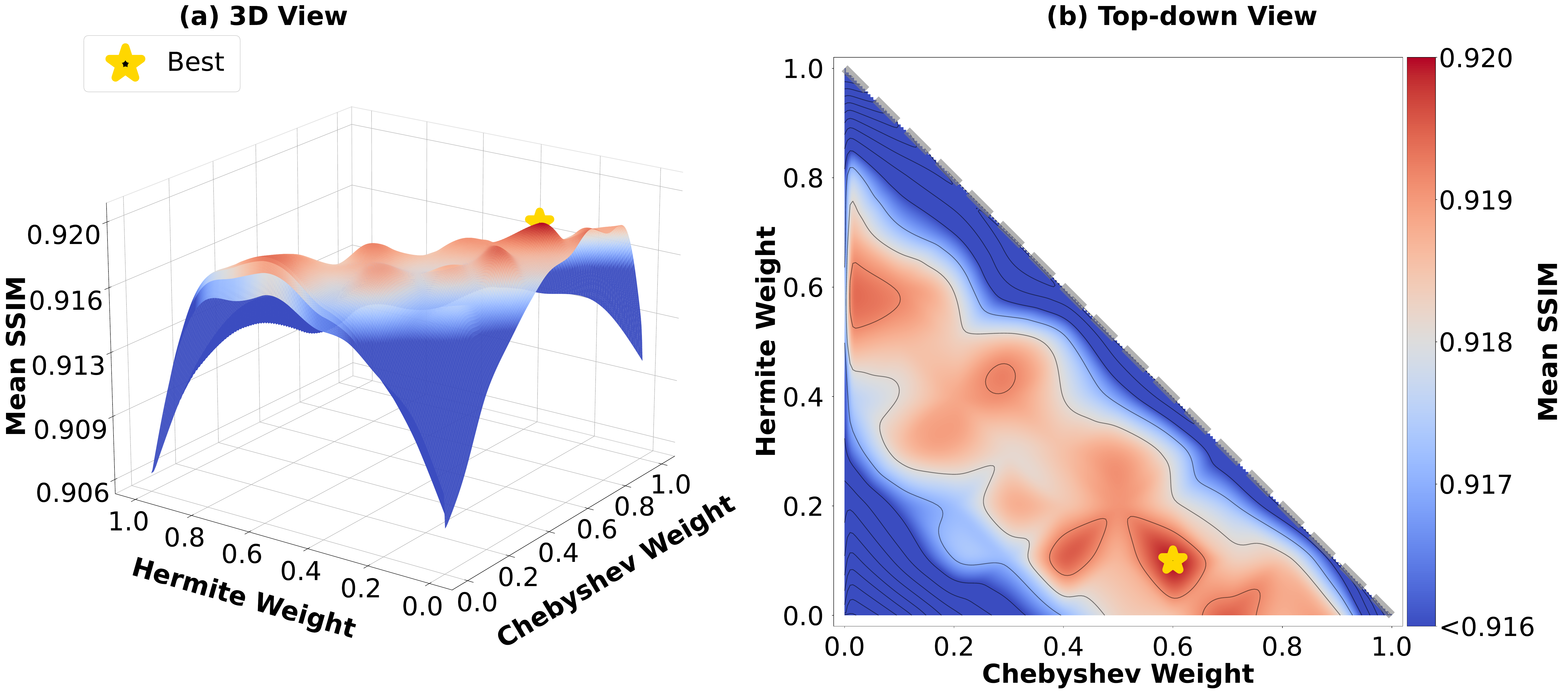}
\end{overpic}
\caption{SSIM landscape of HybridFF over the basis-weight simplex on a small FLUX.1-dev calibration set. The optimum (yellow star) lies inside the simplex at $(\pi^{(\mathrm{Chebyshev})},\pi^{(\mathrm{Hermite})},\pi^{(\mathrm{Taylor})})=(0.6,0.1,0.3)$.}
\label{fig:ssim-heatmap}
\end{figure}

\subsection{More Qualitative Results}\label{sec:appendix_more_qualitative}

This subsection presents additional qualitative comparisons for both HybridFF(F) and HybridFF(A), where F/A denote fixed/adaptive fusion weights.

\begin{figure*}[!t]
\centering
\includegraphics[width=0.485\textwidth]{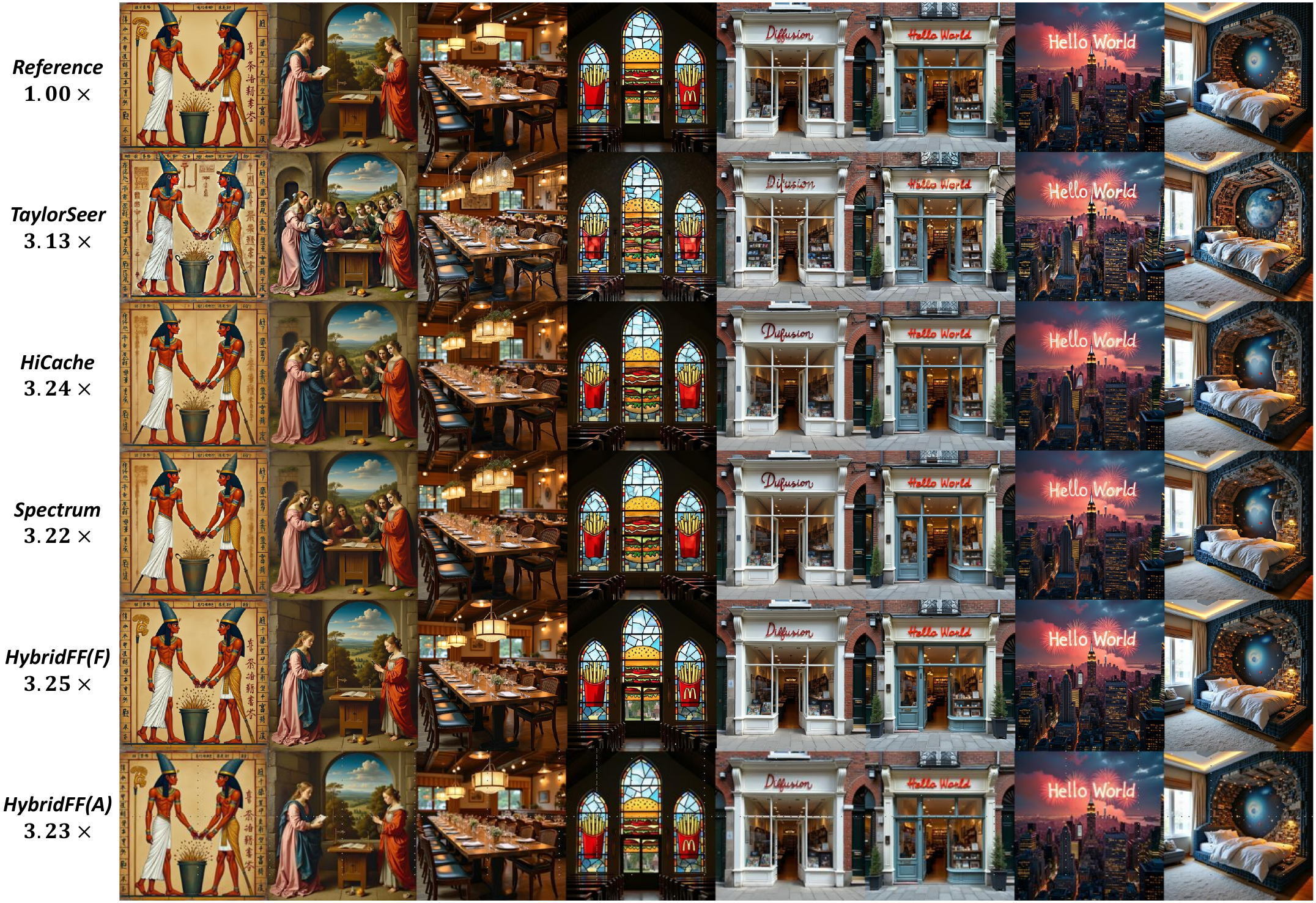}\hfill
\includegraphics[width=0.485\textwidth]{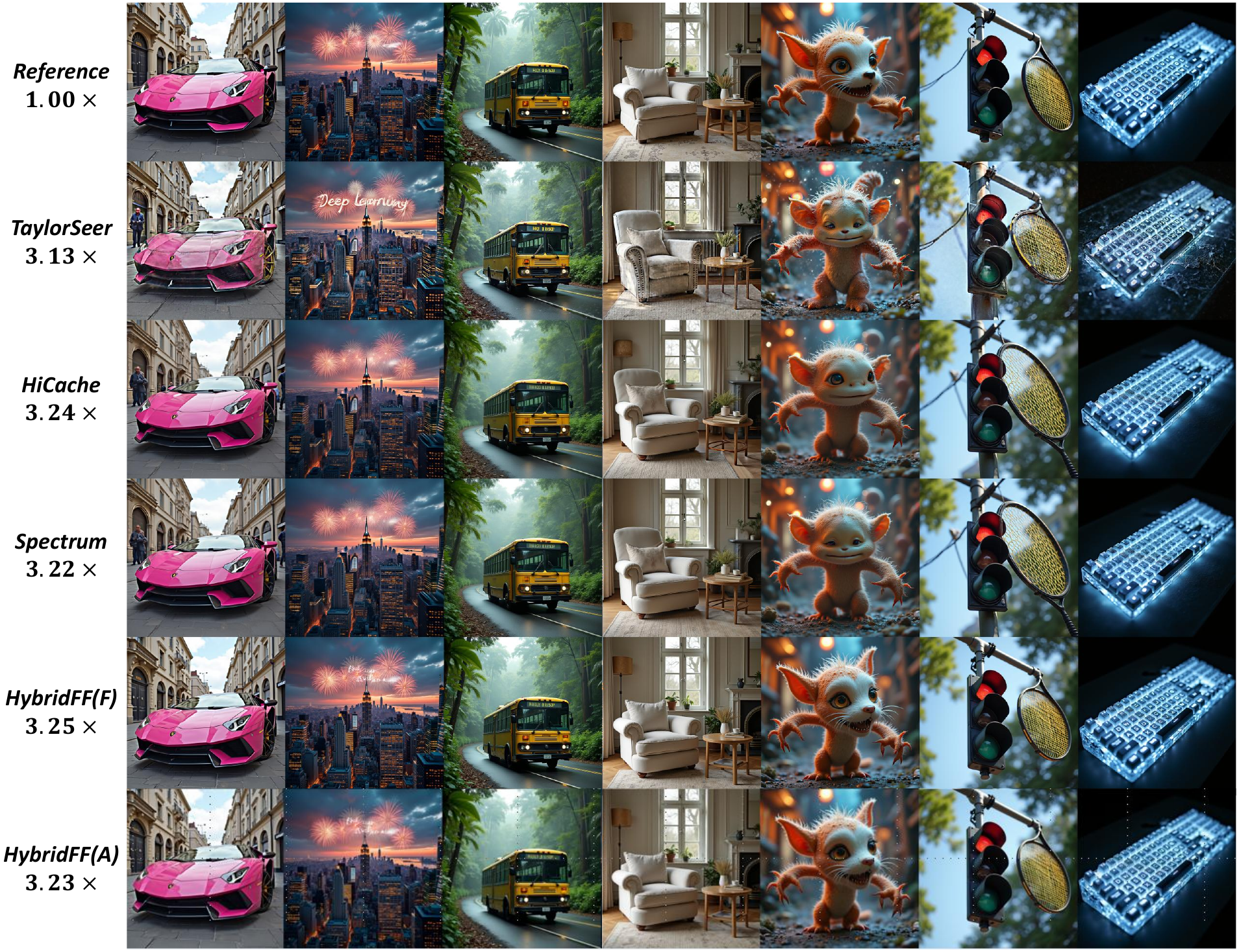}
\caption{Additional qualitative comparisons on FLUX.1-dev at NFE\,$=$\,10. The left and right panels show two prompt sets. F/A denote fixed/adaptive fusion weights.}
\label{fig:appendix_flux_qual}
\end{figure*}

\paragraph{Text-to-image (FLUX.1-dev).}
Figure~\ref{fig:appendix_flux_qual} reports two independent prompt sets at NFE\,$=$\,10. Both HybridFF variants better preserve the prompted subjects, object counts, and fine detail than the single-basis baselines, with HybridFF(A) giving the closest match to the 50-step reference in difficult regions.

\paragraph{Text-to-video (HunyuanVideo).}
Figure~\ref{fig:appendix_hunyuan_qual} shows a representative clip generated at NFE\,$=$\,10. HybridFF(F) keeps the strongest acceleration, while HybridFF(A) stays closest to the 50-step reference across frames and better preserves geometry and surface texture.

\begin{figure*}[!t]
\centering
\includegraphics[width=0.80\textwidth]{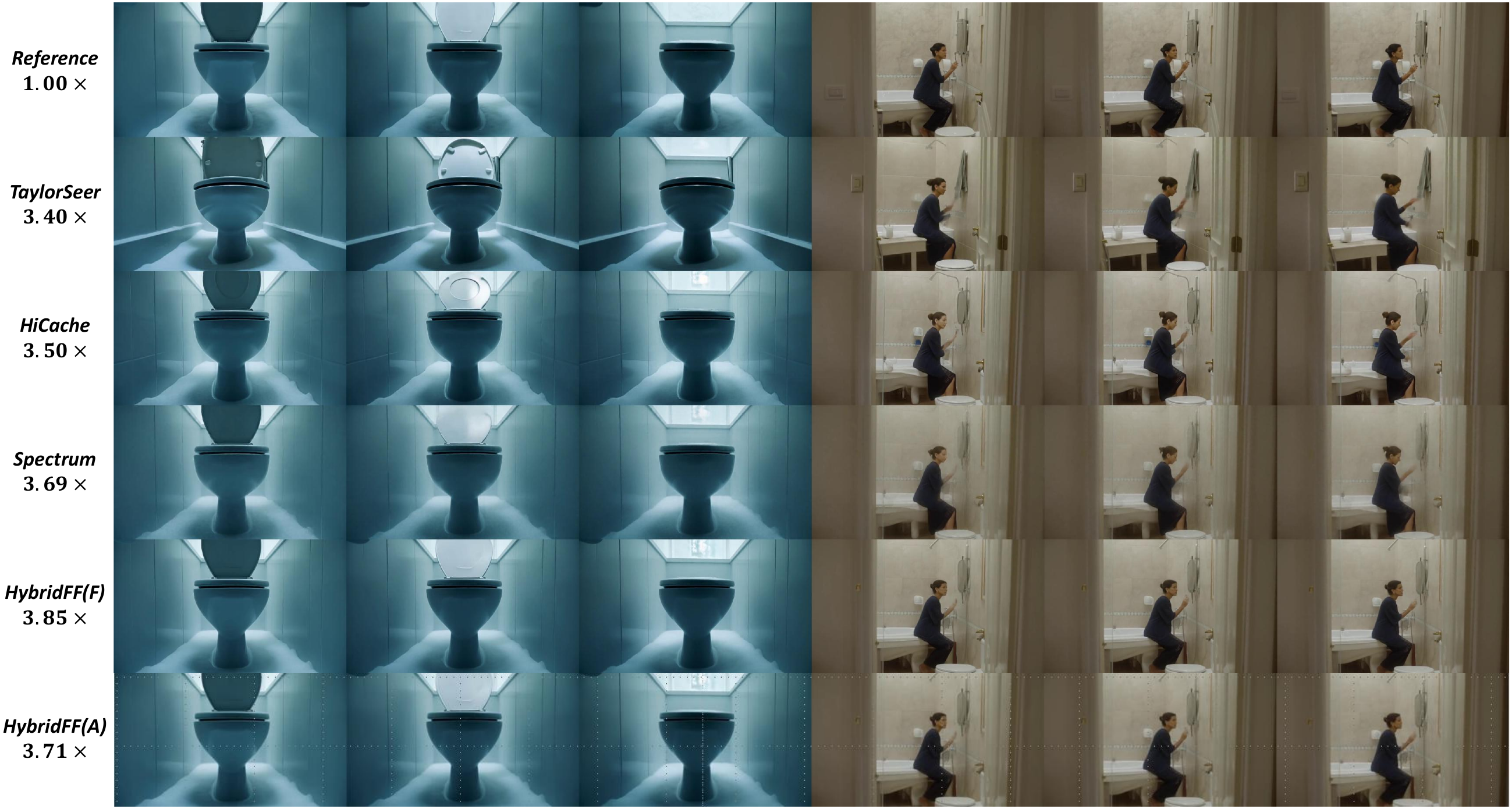}
\caption{Additional qualitative comparison on HunyuanVideo at NFE\,$=$\,10. F/A denote fixed/adaptive fusion weights.}
\label{fig:appendix_hunyuan_qual}
\end{figure*}

%% file: src_combination/conclusion.tex
\section{Conclusion}\label{sec:conclusion}

We presented \textbf{HybridFF}, a training-free, plug-and-play feature forecaster that accelerates diffusion sampling by combining complementary temporal bases within a unified MLS framework. {The method is motivated by the empirical observation that intermediate transformer features follow structured trajectories along denoising steps, making skipped-step features predictable from cached full-step observations.} Fusion happens in coefficient space on the simplex with calibrated weights, followed by an anchored residual reconstruction; the predictor is linear, dimension-free, much cheaper than a full block evaluation, and provably no worse than the best single branch.
We also design two weighting strategies: a Fixed variant calibrated offline for speed, and an Adaptive variant updated online from full-step errors for fidelity.

 Across DiT, FLUX.1-dev, SD3.5-Large, and HunyuanVideo, HybridFF dominates single-basis baselines on every primary fidelity metric, and the optimal mixture varies substantially across backbones, confirming that no fixed temporal prior is universally best. HybridFF (Fixed) emphasizes latency reduction, while HybridFF (Adaptive) improves fidelity by updating fusion weights from full-step feature prediction errors without extra denoiser evaluations. 
 
The current study leaves broader basis libraries, schedule-aware full-step/predict-step policies, and finer reliability estimation underexplored. Future work can study additional bases such as rational Pad\'{e} approximants~\cite{cui2026tcpade}, jointly tune refresh schedules and fusion weights under latency budgets, and combine HybridFF with step-reduction samplers.

%% file: src_combination/reference.bib
@inproceedings{peebles2023scalable,
  author = {Peebles, William and Xie, Saining},
  title = {Scalable Diffusion Models with Transformers},
  booktitle = {Proceedings of the IEEE/CVF International Conference on Computer Vision},
  pages = {4172--4182},
  year = {2023}
}

@inproceedings{chen2024pixart,
  author = {Chen, Junsong and Yu, Jincheng and Ge, Chongjian and others},
  title = {{PixArt-$\alpha$}: Fast Training of Diffusion Transformer for Photorealistic Text-to-Image Synthesis},
  booktitle = {International Conference on Learning Representations (ICLR)},
  year = {2024}
}

@misc{kong2024hunyuanvideo,
  author = {Weijie Kong and Qi Tian and Zijian Zhang and Rox Min and Zuozhuo Dai and others},
  title = {{HunyuanVideo}: A Systematic Framework for Large Video Generative Models},
  year = {2024},
  eprint = {2412.03603},
  archivePrefix = {arXiv},
  primaryClass = {cs.CV}
}

@inproceedings{liu2025teacache,
  author = {Liu, Feng and Zhang, Shiwei and Wang, Xiaofeng and Wei, Yujie and Qiu, Haonan and Zhao, Yuzhong and Zhang, Yingya and Ye, Qixiang and Wan, Fang},
  title = {Timestep Embedding Tells: It's Time to Cache for Video Diffusion Model},
  booktitle = {Proceedings of the IEEE/CVF Conference on Computer Vision and Pattern Recognition (CVPR)},
  pages = {7353--7363},
  year = {2025}
}

@misc{zhu2026tap,
  author = {Haowei Zhu and Tingxuan Huang and Xing Wang and Tianyu Zhao and Jiexi Wang and Weifeng Chen and Xurui Peng and Fangmin Chen and Junhai Yong and Bin Wang},
  title = {{TAP}: A Token-Adaptive Predictor Framework for Training-Free Diffusion Acceleration},
  year = {2026},
  eprint = {2603.03792},
  archivePrefix = {arXiv},
  primaryClass = {cs.CV}
}

@inproceedings{lu2022dpm,
  author = {Lu, Cheng and Zhou, Yuhao and Bao, Fan and Chen, Jianfei and LI, Chongxuan and Zhu, Jun},
  title = {{DPM-Solver}: A Fast {ODE} Solver for Diffusion Probabilistic Model Sampling in Around 10 Steps},
  booktitle = {Advances in Neural Information Processing Systems (NeurIPS)},
  volume = {35},
  pages = {5775--5787},
  year = {2022}
}

@inproceedings{song2023consistency,
  author = {Song, Yang and Dhariwal, Prafulla and Chen, Mark and Sutskever, Ilya},
  title = {Consistency Models},
  booktitle = {Proceedings of the 40th International Conference on Machine Learning (ICML)},
  year = {2023}
}

@inproceedings{li2024faster,
  author = {Li, Senmao and Hu, Taihang and van de Weijer, Joost and Khan, Fahad S and Liu, Tao and Li, Linxuan and Yang, Shiqi and Wang, Yaxing and Cheng, Ming-Ming and Yang, Jian},
  title = {Faster Diffusion: Rethinking the Role of the Encoder for Diffusion Model Inference},
  booktitle = {Advances in Neural Information Processing Systems (NeurIPS)},
  volume = {37},
  pages = {85203--85240},
  year = {2024}
}

@inproceedings{fang2023pruning,
  author = {Fang, Gongfan and Ma, Xinyin and Wang, Xinchao},
  title = {Structural Pruning for Diffusion Models},
  booktitle = {Advances in Neural Information Processing Systems},
  volume = {36},
  pages = {16716--16728},
  year = {2023}
}

@inproceedings{ma2024deepcache,
  author = {Ma, Xinyin and Fang, Gongfan and Wang, Xinchao},
  title = {{DeepCache}: Accelerating Diffusion Models for Free},
  booktitle = {Proceedings of the IEEE/CVF Conference on Computer Vision and Pattern Recognition (CVPR)},
  pages = {15762--15772},
  year = {2024},
}

@inproceedings{liu2025taylorseer,
  author = {Liu, Jiacheng and Zou, Chang and Lyu, Yuanhuiyi and Chen, Junjie and Zhang, Linfeng},
  title = {From Reusing to Forecasting: Accelerating Diffusion Models with {TaylorSeers}},
  booktitle = {Proceedings of the IEEE/CVF International Conference on Computer Vision (ICCV)},
  pages = {15853--15863},
  year = {2025}
}

@article{yu2025abcache,
  author = {Yu, Zichao and Zou, Zhen and Shao, Guojiang and Zhang, Chenwei and Xu, Shengze and Huang, Jie and Zhao, Feng and Cun, Xiaodong and Zhang, Wenyi},
  title = {{AB-Cache}: Training-Free Acceleration of Diffusion Models via {Adams-Bashforth} Cached Feature Reuse},
  journal = {Proceedings of the 33rd ACM International Conference on Multimedia},
  year = {2025}
}

@inproceedings{feng2026hicache,
  author = {Feng, Liang and Zheng, Shikang and Liu, Jiacheng and Lin, Yuqi and Zhou, Qinming and Cai, Peiliang and Wang, Xinyu and Chen, Junjie and Zou, Chang and Ma, Yue and others},
  title = {{HiCache}: A Plug-in {Scaled-Hermite} Upgrade for {Taylor}-Style Cache-then-Forecast Diffusion Acceleration},
  booktitle = {The Fourteenth International Conference on Learning Representations (ICLR)},
  year = {2026}
}

@inproceedings{han2026spectrum,
  author = {Han, Jiaqi and Shi, Juntong and Li, Puheng and Ye, Haotian and Guo, Qiushan and Ermon, Stefano},
  title = {Adaptive Spectral Feature Forecasting for Diffusion Sampling Acceleration},
  booktitle = {Proceedings of the IEEE/CVF Conference on Computer Vision and Pattern Recognition (CVPR)},
  year = {2026}
}

@inproceedings{cui2026tcpade,
  author = {Cui, Benlei and He, Shaoxuan and Huang, Bukun and Ye, Zhizeng and Sun, Yunyun and Huang, Longtao and Xue, Hui and Yang, Yang and Tang, Jingqun and Zhao, Zhou and others},
  title = {{TC-Pad\'{e}}: Trajectory-Consistent {Pad\'{e}} Approximation for Diffusion Acceleration},
  booktitle = {Proceedings of the IEEE/CVF Conference on Computer Vision and Pattern Recognition (CVPR)},
  year = {2026}
}

@misc{cai2026lesa,
  author = {Cai, Peiliang and Liu, Jiacheng and Xu, Haowen and Wang, Xinyu and Zou, Chang and Zhang, Linfeng},
  title = {{LESA}: Learnable Stage-Aware Predictors for Diffusion Model Acceleration},
  year = {2026},
  eprint = {2602.20497},
  archivePrefix = {arXiv},
  primaryClass = {cs.CV}
}

@misc{wang2026zeus,
  author = {Wang, Yixiao and Jiang, Ting and Shao, Zishan and others},
  title = {{ZEUS}: Accelerating Diffusion Models with Only Second-Order Predictor},
  year = {2026},
  eprint = {2604.01552},
  archivePrefix = {arXiv},
  primaryClass = {cs.LG}
}

@inproceedings{zheng2026foca,
  author = {Zheng, Shikang and Feng, Liang and Wang, Xinyu and others},
  title = {Forecast Then Calibrate: Feature Caching as {ODE} for Efficient Diffusion Transformers},
  booktitle = {Proceedings of the AAAI Conference on Artificial Intelligence},
  volume = {40},
  pages = {13449--13457},
  year = {2026}
}

@article{ho2020ddpm,
  author = {Ho, Jonathan and Jain, Ajay and Abbeel, Pieter},
  title = {Denoising Diffusion Probabilistic Models},
  journal = {Advances in Neural Information Processing Systems},
  volume = {33},
  pages = {6840--6851},
  year = {2020}
}

@inproceedings{rombach2022ldm,
  author = {Rombach, Robin and Blattmann, Andreas and Lorenz, Dominik and others},
  title = {High-Resolution Image Synthesis with Latent Diffusion Models},
  booktitle = {Proceedings of the IEEE/CVF Conference on Computer Vision and Pattern Recognition (CVPR)},
  year = {2022}
}

@inproceedings{saharia2022imagen,
  author = {Saharia, Chitwan and Chan, William and Saxena, Saurabh and others},
  title = {Photorealistic Text-to-Image Diffusion Models with Deep Language Understanding},
  booktitle = {Advances in Neural Information Processing Systems (NeurIPS)},
  volume = {35},
  pages = {36479--36494},
  year = {2022}
}

@misc{podell2023sdxl,
  author = {Podell, Dustin and English, Zion and Lacey, Kyle and others},
  title = {{SDXL}: Improving Latent Diffusion Models for High-Resolution Image Synthesis},
  year = {2023},
  eprint = {2307.01952},
  archivePrefix = {arXiv},
  primaryClass = {cs.CV}
}

@misc{esser2024flux,
  author = {{Black Forest Labs}},
  title = {{FLUX.1}: A Suite of Open Image Generation Models},
  year = {2024},
  howpublished = {\url{https://github.com/black-forest-labs/flux}},
  note = {Accessed 2024}
}

@inproceedings{ho2022videodiffusion,
  author = {Ho, Jonathan and Salimans, Tim and Gritsenko, Alexey and others},
  title = {Video Diffusion Models},
  booktitle = {Advances in Neural Information Processing Systems (NeurIPS)},
  volume = {35},
  pages = {8633--8646},
  year = {2022}
}

@inproceedings{singer2022makeavideo,
  author = {Singer, Uriel and Polyak, Adam and Hayes, Thomas and others},
  title = {{Make-A-Video}: Text-to-Video Generation Without Text-Video Data},
  booktitle = {The Eleventh International Conference on Learning Representations (ICLR)},
  year = {2023}
}

@misc{blattmann2023,
  author = {Blattmann, Andreas and Dockhorn, Tim and Kulal, Sumith and others},
  title = {Stable Video Diffusion: Scaling Latent Video Diffusion Models to Large Datasets},
  year = {2023},
  eprint = {2311.15127},
  archivePrefix = {arXiv},
  primaryClass = {cs.CV}
}

@article{saharia2021srdiff,
  author = {Saharia, Chitwan and Ho, Jonathan and Chan, William and others},
  title = {Image Super-Resolution via Iterative Refinement},
  journal = {IEEE Transactions on Pattern Analysis and Machine Intelligence},
  volume = {45},
  number = {4},
  pages = {4713--4726},
  year = {2023},
  doi = {10.1109/TPAMI.2022.3204461}
}

@inproceedings{lugmayr2022repaint,
  author = {Lugmayr, Andreas and Danelljan, Martin and Romero, Andr{\'e}s and others},
  title = {{RePaint}: Inpainting Using Denoising Diffusion Probabilistic Models},
  booktitle = {Proceedings of the IEEE/CVF Conference on Computer Vision and Pattern Recognition (CVPR)},
  pages = {11451--11461},
  year = {2022}
}

@inproceedings{song2021ddim,
  author = {Song, Jiaming and Meng, Chenlin and Ermon, Stefano},
  title = {Denoising Diffusion Implicit Models},
  booktitle = {International Conference on Learning Representations (ICLR)},
  year = {2021}
}

@inproceedings{salimans2022progressive,
  author = {Salimans, Tim and Ho, Jonathan},
  title = {Progressive Distillation for Fast Sampling of Diffusion Models},
  booktitle = {International Conference on Learning Representations (ICLR)},
  year = {2022}
}

@inproceedings{shang2023ptqdiffusion,
  author = {Shang, Yuzhang and Yuan, Zhihang and Xie, Bin and others},
  title = {Post-Training Quantization on Diffusion Models},
  booktitle = {Proceedings of the IEEE/CVF Conference on Computer Vision and Pattern Recognition (CVPR)},
  pages = {1972--1981},
  year = {2023},
}

@inproceedings{bolya2023tokenmerging,
  author = {Bolya, Daniel and Hoffman, Judy},
  title = {Token Merging for Fast Stable Diffusion},
  booktitle = {IEEE/CVF Conference on Computer Vision and Pattern Recognition Workshops (CVPRW)},
  pages = {4599--4603},
  year = {2023},
  doi = {10.1109/CVPRW59228.2023.00484}
}

@inproceedings{li2023qdiffusion,
  author = {Li, Xiuyu and Liu, Yijiang and Lian, Long and others},
  title = {{Q-Diffusion}: Quantizing Diffusion Models},
  booktitle = {Proceedings of the IEEE/CVF International Conference on Computer Vision (ICCV)},
  pages = {17489--17499},
  year = {2023},
  doi = {10.1109/ICCV51070.2023.01608}
}

@inproceedings{yuan2024ditfastattn,
  author = {Yuan, Zhihang and Zhang, Hanling and Lu, Pu and others},
  title = {{DiTFastAttn}: Attention Compression for Diffusion Transformer Models},
  booktitle = {Advances in Neural Information Processing Systems (NeurIPS)},
  volume = {37},
  pages = {1196--1219},
  year = {2024}
}

@misc{selvaraju2024fora,
  author = {Selvaraju, Pratheba and Ding, Tianyu and Chen, Tianyi and others},
  title = {{FORA}: Fast-Forward Caching in Diffusion Transformer Acceleration},
  year = {2024},
  eprint = {2407.01425},
  archivePrefix = {arXiv},
  primaryClass = {cs.CV}
}

@inproceedings{zou2025toca,
  author = {Zou, Chang and Liu, Xuyang and Liu, Ting and others},
  title = {Accelerating Diffusion Transformers with Token-Wise Feature Caching},
  booktitle = {The Thirteenth International Conference on Learning Representations (ICLR)},
  year = {2025}
}

@inproceedings{liu2025smoothcache,
  author = {Liu, Joseph and Geddes, Joshua and Guo, Ziyu and others},
  title = {{Smooth Cache}: A Universal Inference Acceleration Technique for Diffusion Transformers},
  booktitle = {IEEE/CVF Conference on Computer Vision and Pattern Recognition Workshops (CVPRW)},
  pages = {3220--3229},
  year = {2025},
}

@inproceedings{sun2025unicp,
  author = {Sun, Wenzhang and Hou, Qirui and Di, Donglin and others},
  title = {{UniCP}: A Unified Caching and Pruning Framework for Efficient Video Generation},
  booktitle = {Proceedings of the 7th ACM International Conference on Multimedia in Asia},
  year = {2025},
  doi = {10.1145/3743093.3770981}
}

@article{liu2025speca,
  author = {Liu, Jiacheng and Zou, Chang and Lyu, Yuanhuiyi and others},
  title = {{SpeCa}: Accelerating Diffusion Transformers with Speculative Feature Caching},
  journal = {Proceedings of the 33rd ACM International Conference on Multimedia},
  year = {2025}
}

@misc{chen2026svdcache,
  author = {Chen, Guantao and Zheng, Shikang and Lin, Yuqi and others},
  title = {Forecast the Principal, Stabilize the Residual: Subspace-Aware Feature Caching for Efficient Diffusion Transformers},
  year = {2026},
  eprint = {2601.07396},
  archivePrefix = {arXiv},
  primaryClass = {cs.CV}
}

@misc{zou2026disca,
  author = {Zou, Chang and Li, Changlin and Li, Yang and others},
  title = {{DisCa}: Accelerating Video Diffusion Transformers with Distillation-Compatible Learnable Feature Caching},
  year = {2026},
  eprint = {2602.05449},
  archivePrefix = {arXiv},
  primaryClass = {cs.CV}
}

@misc{esser2024,
      title={Scaling Rectified Flow Transformers for High-Resolution Image Synthesis}, 
      author={Patrick, Esser and Sumith, Kulal and Andreas, Blattmann and others},
      year={2024},
      eprint={2403.03206},
      archivePrefix={arXiv},
      primaryClass={cs.CV},
      url={https://arxiv.org/abs/2403.03206}, 
}
